\documentclass{article} 
\usepackage{iclr2026_conference,times}

\usepackage{amsmath,amsfonts,bm}
\usepackage{xargs}

\def\1{\bm{1}}

\newcommand{\test}{\mathcal{D_{\mathrm{test}}}}

\DeclareMathAlphabet{\mathsfit}{\encodingdefault}{\sfdefault}{m}{sl}
\SetMathAlphabet{\mathsfit}{bold}{\encodingdefault}{\sfdefault}{bx}{n}

\newcommand{\E}{\mathbb{E}}

\newcommand{\rset}{\mathbb{R}}

\usepackage{hyperref}
\usepackage{amsthm}
\usepackage{amsmath}
\usepackage{amsfonts}
\usepackage{amssymb}
\usepackage{mathtools}
\usepackage{nicefrac}   
\usepackage{url}
\usepackage{xcolor}
\usepackage{xcolor}
\usepackage{booktabs}
\usepackage{enumitem}
\setlist[enumerate]{noitemsep, topsep=0pt,left=0pt, parsep=0pt}
\usepackage{cleveref}
\usepackage{bbm}
\usepackage{algorithm}
\usepackage{algpseudocode}
\usepackage{amsmath}
\usepackage{comment}
\usepackage{multirow}

\usepackage[colorinlistoftodos,backgroundcolor=yellow!20,bordercolor=red,linecolor=blue]{todonotes}

\newcommand{\ecoparagraph}[1]{\vspace{0.1cm}\noindent\textbf{#1}}

\newtheorem{assumption}{Assumption}
\crefname{assumption}{assumption}{Assumption}

\def\EE{\mathbb{E}}
\def\PP{\mathbb{P}}
\def\RR{\mathbb{R}}
\def\AC{\mathcal{A}}
\def\BC{\mathcal{B}}
\def\CC{\mathcal{C}}
\def\DC{\mathcal{D}}

\def\LC{\mathcal{L}}
\def\MC{\mathcal{M}}
\def\NC{\mathcal{N}}

\def\QC{\mathcal{Q}}
\def\UC{\mathcal{U}}
\def\VC{\mathcal{V}}
\def\XC{\mathcal{X}}
\def\YC{\mathcal{Y}}

\def\dimy{d_y}
\def\dimx{d_x}
\def\Id{\operatorname{I}}
\def\vol{\operatorname{Vol}}

\title{Neural Optimal Transport Meets \\ Multivariate Conformal Prediction}

\author{%
  Vladimir Kondratyev \\
  Department of ML, MBZUAI, UAE \\
  \texttt{vladimir.kondratyev@mbzuai.ac.ae} \\
  \And
  Alexander Fishkov \\
  Department of ML, MBZUAI, UAE
  \And
  Nikita Kotelevskii \\
  Department of ML, MBZUAI, UAE
  \And
  Mahmoud Hegazy \\
  CMAP, École polytechnique, France
  \And
  Rémi Flamary \\
  CMAP, Ecole Polytechnique, Palaiseau, France
  \And
  Maxim Panov \\
  Department of ML, MBZUAI, UAE
  \And
  Eric Moulines \\
  CMAP, Ecole Polytechnique, Palaiseau, France \\
  \texttt{eric.moulines@mbzuai.ac.ae} \\
}

\iclrfinalcopy 
\definecolor{darkgreen}{rgb}{0.0,0.5,0.0}

\newtheorem{theorem}{Theorem}

\newcommand{\open}[1]{\left ( #1 \right )}
\newcommand{\closed}[1]{\left [#1 \right]}
\newcommand{\ens}[1]{\left \{ #1 \right \} }

\definecolor{asparagus}{rgb}{0.53, 0.66, 0.42}

\def\paramamor{\vartheta}
\def\rmd{\mathrm{d}}

\newcommandx{\ball}[3][1=]{\operatorname{B}(#2,#3)}

\def\pullback{\mathrm{pb}}

\def\test{\mathrm{train}}
\def\test{\mathrm{test}}
\def\cal{\mathrm{cal}}

\def\unipullback{\mathrm{rpb}}

\theoremstyle{remark}
\newtheorem{remark}{Remark}

\begin{document}

\maketitle

\begin{abstract}
  We propose a framework for \emph{conditional vector quantile regression} (CVQR) that combines neural optimal transport with amortized optimization, and apply it to multivariate conformal prediction. Classical quantile regression does not extend naturally to multivariate responses, while existing approaches often ignore the geometry of joint distributions. Our method parameterizes the conditional vector quantile function as the gradient of a convex potential implemented by an input-convex neural network, ensuring monotonicity and uniform ranks. To reduce the cost of solving high-dimensional variational problems, we introduce amortized optimization of the dual potentials, yielding efficient training and faster inference.  

  We then exploit the induced multivariate ranks for conformal prediction, constructing distribution-free predictive regions with finite-sample validity. Unlike coordinatewise methods, our approach adapts to the geometry of the conditional distribution, producing tighter and more informative regions. Experiments on benchmark datasets show improved coverage–efficiency trade-offs compared to baselines, highlighting the benefits of integrating neural optimal transport with conformal prediction.
\end{abstract}

\section{Introduction}
\label{sec:introduction}

Quantile regression has long been a cornerstone for modeling heterogeneous conditional distributions in the univariate setting~\citep{koenker1978regression,koenker2005quantile}. Estimating conditional quantiles rather than conditional means provides a more complete view of the conditional law of a response variable and has enabled advances in econometrics, statistics, and machine learning. Extending these ideas to multivariate responses, however, remains challenging: unlike the scalar case, $\RR^d$ lacks a natural total ordering, and early multivariate notions of quantiles, based on projections, spatial medians, or depth functions, inherit only part of the desirable scalar properties~\citep{chaudhuri1996geometric,hallin2021distribution}.

Recent progress in optimal transport has offered a principled definition of multivariate ranks and quantiles~\citep{chernozhukov2017monge,hallin2024multivariate}. By interpreting quantiles as transport maps from a reference distribution to the law of $Y$, these approaches recover distribution-free center-outward ranks and quantile regions that extend univariate order statistics to high dimensions. Building on this perspective, vector quantile regression (VQR; \citealp{carlier2016vector,carlier2017vector}) introduces conditional vector quantile functions (CVQFs), monotone maps that represent $Y$ as a transformation of latent uniform variables given covariates. CVQFs provide a rich yet tractable representation of conditional distributions, with promising extensions to nonlinear models~\citep{rosenberg2023fast,tallini2023continuous,del2025nonparametric}. However, practical estimation remains computationally demanding, often requiring large-scale optimal transport solvers.

In parallel, conformal prediction has emerged as a powerful framework for constructing predictive regions with finite-sample coverage guarantees~\citep{angelopoulos2023conformal}. While well studied in the univariate case, multivariate extensions are less developed and often reduce to coordinatewise methods that ignore the geometry of joint distributions~\citep{dheur2025unified}. Very recent work has begun to bridge this gap by incorporating optimal transport–based multivariate ranks into conformal prediction, yielding theoretically grounded multivariate prediction sets~\citep{thurin2025optimal,klein2025multivariate}.

In this paper, we leverage a neural optimal transport framework for learning CVQFs which allows to estimate  parametric cyclically monotone vector quantiles and multivariate ranks. Building on the resulting multivariate ranks, we use conformal prediction to produce distribution-free valid confidence regions that adapt to the geometry of conditional distributions in the multivariate setting. 

We make three main contributions:
\begin{enumerate}
  \item We present a \emph{neural optimal transport} framework for conditional vector quantile regression (CVQR), which utilities input-convex neural networks to estimate continuous vector quantile maps and multidimensional ranks; see \Cref{sec:Neural_ot_VQR}.

  \item We establish a principled integration of multivariate ranks and vector quantiles into conformal prediction, producing distribution-free predictive regions that adapt to the geometry of conditional distributions; see \Cref{sec:conformal_ot_neural_maps}.

  \item We experimentally show that amortized optimization yields gains in training and inference efficiency, while preserving the convexity and monotonicity guarantees of vector quantile functions; see \Cref{sec:optimal_transport_experiments}. The resulting conformal prediction sets outperform coordinatewise and representation-based baselines; see \Cref{sec:conformal_experiments}.
\end{enumerate}


\section{Constructing Multivariate Confidence Sets}
\label{sec:background}

	We start by informally introducing the conditional vector quantile and rank maps that aim to provide a flexible representation of the conditional law of $Y$ given $X$. 

\ecoparagraph{Quantiles in 1D and Confidence Sets.}
	Let us first consider the case of $Y \in \YC \subseteq \RR$. Let $(Y, X) \sim F_{YX}$ and let $F_{Y \mid X}$ be the conditional distribution of $Y$ given $X$. Then, the quantile function $Q_{Y \mid X}(\cdot, x)$ for any $\alpha \in [0, 1]$ outputs the corresponding quantile value $Q_{Y \mid X}(\alpha, x) \in \YC$ of distribution $F_{Y \mid X = x}$. The knowledge of the quantile function is instrumental for the construction of the confidence sets. For example, for a given $\alpha \in (0, 1)$ one can define $\CC_{\alpha}(x) = [Q_{Y \mid X}(\alpha / 2, x), Q_{Y \mid X}(1 - \alpha / 2, x)]$. By construction, this confidence set is valid, i.e. $\PP(Y \in \CC_{\alpha}(x) \mid X = x) = 1 - \alpha$.

	The inverse map $Q^{-1}_{Y \mid X}$ is sometimes called a rank function as for any value of variable $y$ it produces the value on an interval $Q^{-1}_{Y \mid X}(y, x) \in [0, 1]$ which can be interpreted as the rank of $y$ among its possible values with respect to the distribution $F_{Y \mid X = x}$. Importantly, the distribution of $Q^{-1}_{Y \mid X}(Y, X) \mid X = x$ is uniform on $[0, 1]$. In its turn, the knowledge of the rank function gives an alternative way to define the confidence interval $\CC^{\mathrm{pull}}_{\alpha}(x) = \{y\colon Q^{-1}_{Y \mid X}(y, x) \in [\alpha / 2, 1 - \alpha / 2]\}$. Obviously, $\CC_{\alpha}(x)$ and $\CC^{\mathrm{pull}}_{\alpha}(x)$ coincide. However, their functional forms give alternative views on how one can construct the confidence interval depending on having the access to the quantile or to the rank function.

\ecoparagraph{Multivariate Quantiles.}
  In the absence of a natural order on $\RR^d$ for $d > 1$, the definition of the multivariate quantile is not trivial. In this paper, we will study the definitions of multivariate quantiles based on optimal transport; see among others~\citep{carlier2016vector,hallin2021distribution,hallin2024multivariate}. We start by looking at a specific example, while the full exposition in \Cref{sec:conditional-VQR} is given below.

  Define $r_{1 - \alpha} \in \RR_+$ such that the Euclidean ball $\ball[d_y]{0}{r_{1-\alpha}} \subset \UC := \ball[d_y]{0}{1}$ satisfies the condition $\operatorname{Volume}(\ball[d_y]{0}{r_{1-\alpha}}) = 1 - \alpha$. Then, it can be shown (see \Cref{thm:carlier-existence} below) that there exists a map $Q_{Y \mid X}(u, x)$ and  a uniform random variable $U$ over the $\UC$, independent of $X$ such that $Y = Q_{Y \mid X}(U, X)$ almost surely. This map is called a \emph{vector quantile}. The corresponding inverse map $Q^{-1}_{Y \mid X}(y,x) \in \UC$ becomes a natural analogue of the \emph{rank function}.

	We can directly proceed with construction of confidence sets based on $Q^{-1}_{Y \mid X}(Y, X)$. For $x \in \XC$, define the \emph{pullback set}
  \begin{equation}
  \label{eq:pullback_ball}
    \CC^\pullback_{\alpha}(x) \coloneqq \bigl\{y\colon Q^{-1}_{Y \mid X}(y, x) \in \ball[d_y]{0}{r_{1-\alpha}}\bigr\}.
  \end{equation}
  Using the properties of quantile and rank functions we get that 
  \[
  	\PP(Y \in \CC^\pullback_{\alpha}(X))= \PP_{(U, X) \sim F_U \otimes F_X}(\|Q^{-1}_{Y \mid X}(Q_{Y \mid X}(U, X), X)\| \leq r_{1-\alpha}) = \PP_{U \sim F_U}(\|U\| \leq r_{1-\alpha}).
  \]
  Hence, the coverage of the pullback set $\CC^\pullback_{\alpha}(x)$ is exactly $1-\alpha$ as required.

\ecoparagraph{Conformalized Confidence Sets.}
	In practice, we can only have access to the estimator $\widehat{Q}^{-1}_{Y\mid X}$ of $Q^{-1}_{Y\mid X}$.
  One can consider plug-in confidence sets constructed directly from these estimators. However, such sets fail to guarantee coverage as generally $\widehat{Q}^{-1}_{Y\mid X} \neq Q^{-1}_{Y\mid X}$. Consequently, the coverage of $\CC^{\mathrm{pull}}_{\alpha}(X)$ may be miscalibrated, motivating the use of conformal prediction. Conformal prediction corrects such miscalibration, providing finite-sample, distribution-free \emph{marginal} coverage guarantees. Specifically, given a calibration set $\DC_{\mathrm{cal}}=\{(X_i,Y_i)\}_{i=1}^n$ independent of the training data, consider a score \(S_i = \|\widehat{Q}^{-1}_{Y\mid X}(Y_i,X_i)\|, i = 1, \dots, n\). Then, split-conformal prediction constructs a set $\hat{\CC}^\pullback_{\alpha}(X_{\mathrm{test}})\subseteq \YC$ for a new test point $(X_{\mathrm{test}}, Y_{\mathrm{test}})$ based on the scores $\{S_{i}\}_{i = 1}^n$ and $S_{\mathrm{test}} = \|\widehat{Q}^{-1}_{Y\mid X}(Y_{\mathrm{test}}, X_{\mathrm{test}})\|$ such that
  \[
    \PP\{Y_{\mathrm{test}}\in \hat{\CC}^\pullback_{\alpha}(X_{\mathrm{test}})\} \geq 1 - \alpha,
  \]
  under the assumption that $(X_1,Y_1),\dots,(X_n,Y_n),(X_{\mathrm{test}},Y_{\mathrm{test}})$ are exchangeable~\citep{romano2019conformalized,angelopoulos2023conformal}. The other choices of the score are possible, see discussion in Section~\ref{sec:conformal_ot_neural_maps}. 

  In what follows, we discuss various approaches to construct \(\widehat{Q}_{Y\mid X}\) and \(\widehat{Q}^{-1}_{Y\mid X}\) based on neural optimal transport.

\section{Vector Quantile Regression via Optimal Transport}
\label{sec:conditional-VQR}
  We now proceed to recall the mathematical underpinnings of vector quantile regression and multidimensional ranks, where we follow closely the formulation of \citet{carlier2016vector,hallin2021distribution}.  Let $(Y,X)$ be a random vector on a complete probability space $(\Omega,\AC,\PP)$, where $Y\in\RR^{\dimy}$ and $X\in \RR^{\dimx}$. Denote by $F_{YX}$ the joint law of $(Y,X)$, by $F_{Y\mid X}$ the conditional law of $Y$ given $X$, and by $F_X$ the marginal of $X$. Let $U$ be a random vector  on $(\Omega,\AC,\PP)$ with reference distribution $F_U$. We write $\YC, \XC, \UC, \YC \times \XC, \UC \times \XC$ for the supports of $F_Y, F_X, F_U, F_{YX}, F_{UX}$, and $\YC_x$ for the support of $F_{Y\mid X=x}$. Norms are Euclidean on $\RR^d$.

  The following basic properties of distributions $F_U$ and $F_{Y\mid X}$ are required for the construction of OT-based vector quantiles and rank functions.
  \begin{assumption}
  \label{assum:density-reference}
    The reference distribution $F_U$ admits a density $f_U$ with respect to Lebesgue measure on $\RR^d$, with convex support $\UC \subseteq \RR^d$.
  \end{assumption}
  Typical choices for $F_U$ include the uniform distribution on $[0, 1]^d$, the Gaussian $\NC(0, I_d)$, or any distribution on $\RR^d$ with convex support.

  \begin{assumption}
  \label{assum:absolute_continuity_cond_distr}
    For each $x \in \XC$, the conditional law $F_{Y \mid X}(\cdot,x)$ has a density $f_{Y \mid X}(\cdot, x)$.
  \end{assumption}

  Our goal is to construct a push-forward of $U\sim F_U$ to $Y$ such that the conditional law of $Y\mid X$ equals $F_{Y\mid X}$. In the multivariate setting, monotonicity requires the map to be the gradient of a convex function, a natural generalization of scalar monotonicity. This motivates the \emph{conditional vector quantile function} (CVQF).
  \begin{theorem}[\citet{carlier2016vector}, Theorems~2.1 \& 2.2]
  \label{thm:carlier-existence}
    Suppose Assumption~\ref{assum:density-reference} holds. Then:
    \begin{enumerate}[label=(\roman*)]
      \item For each $x\in\XC$, there exists a measurable map $u\mapsto Q_{Y\mid X}(u,x)$, unique $F_U$-a.e., which is the gradient of a convex function and pushes $F_U$ forward to $F_{Y\mid X=x}$.
      
      \item Consequently, there exists $U\sim F_U$ such that $Y=Q_{Y\mid X}(U,X)$ a.s.\ with $U\mid X\sim F_U$.

      \item Additionally, if~ \Cref{assum:absolute_continuity_cond_distr} holds, then there exists a measurable inverse map $Q^{-1}_{Y\mid X}(y,x) \in \UC$ such that $Q^{-1}_{Y\mid X}(Q_{Y\mid X}(u,x),x) = u$ for $F_U$-a.e. $u$, and $\PP(Q^{-1}_{Y\mid X}(Y,X) \le u \mid X=x) = F_U(u)$.
    \end{enumerate}
  \end{theorem}

  %
  The map $y \mapsto Q^{-1}_{Y\mid X}(y,x)$ is the \emph{conditional vector rank}. For $d=1$ it coincides with the conditional CDF, but not for $d>1$~\citep{hallin2021distribution,hallin2024multivariate,del2025nonparametric}.


  Finally, the following assumption is needed to ensure the efficient computation of $Q_{Y \mid X}$ and $Q^{-1}_{Y \mid X}$.
  \begin{assumption}
  \label{assum:finite-2nd-moment}
    $Y$ and $U$ have finite second moments: $\EE[\|Y\|^2]<\infty$ and $\EE[\|U\|^2]<\infty$.
  \end{assumption}

  Under this, the CVQF solves a conditional optimal transport problem:
  $\min_V \EE[\|Y-V\|^2]\quad\text{s.t. }V\mid X\sim F_U$, equivalently $\max_V \EE[V^\top Y]$ under the same constraint. The dual program is
  \begin{equation}
  \label{eq:dual_vqr}
    \min_{\psi,\varphi} \EE[\varphi(V,X)] + \EE[\psi(Y,X)]
    \quad\text{s.t.}\quad \varphi(u,x)+\psi(y,x)\ge u^\top y,
  \end{equation}
  where $V$ is any vector such that \(V \mid X \sim F_U\). The following properties for the solution of~\eqref{eq:dual_vqr} can be stated.

  \begin{theorem}[\citet{carlier2016vector}, Theorem~2.3]
  \label{theorem:existance_of_solution}
    Suppose Assumptions~\ref{assum:density-reference}--\ref{assum:finite-2nd-moment} hold. Then,
    \begin{enumerate}[label=(\roman*)] 
      \item There exist potentials $\varphi(u, x)$ and $\psi(y, x) = \varphi^*(y, x)$ solving~\eqref{eq:dual_vqr}, where for each $x$, $u\mapsto\varphi(u, x)$ and $y\mapsto\psi(y, x)$ are convex and Legendre conjugates:
        \begin{equation}
        \label{eq:legendre-conjugate}
          \varphi(u, x) = \sup_{y \in \YC}\{u^\top y-\psi(y, x)\}, 
          \quad 
          \psi(y, x) = \varphi^*(y, x) =\sup_{u \in \UC}\{u^\top y-\varphi(u, x)\}.
        \end{equation}

      \item The conditional vector quantile is $Q_{Y\mid X}(u, x)=\nabla_u\varphi(u, x)$ for $F_U$-a.e.\ $u$.

      \item The conditional vector rank is $Q_{Y\mid X}^{-1}(y, x)=\nabla_y\psi(y, x)$ for $F_{Y\mid X}(\cdot, x)$-a.e.\ $y$.

      \item These maps are inverses: for each $x$, $\nabla_y\psi(\nabla_u\varphi(u, x), x)=u$, $\nabla_u\varphi(\nabla_y\psi(y, x), x)=y$, for $F_U$-a.e.\ $u$ and $F_{Y\mid X}(\cdot, x)$-a.e.\ $y$. 
    \end{enumerate}
  \end{theorem}
  This theorem gives us necessary tools for the practical solution of OT problem~\eqref{eq:dual_vqr}.

\section{Neural Optimal Transport for VQR}
\label{sec:Neural_ot_VQR}

 We now introduce the proposed approach for learning continuous Neural VQR models. First, we reformulate the optimization problem as a function of a unique (convex) potential using the conditional $c$-transform. We then discuss how this problem can be solved in practice using Partially Input Convex Neural Networks (PICNNs; \citealp{amos2017input}) and how their training can be accelerated by amortized optimization.

\ecoparagraph{Neural parameterization and semi-dual formulation.}
  First, following~\citet{taghvaei20192,makkuva2020optimal,amos2023amortizingconvexconjugatesoptimal}, we propose to reformulate the Monge-Kantorovich dual problem~\eqref{eq:dual_vqr} as an optimization problem over a parametric family of potentials $\varphi_\theta$ with parameters $\theta$. Since $\varphi_\theta$ should be convex in its first argument, it is ensured that one can estimate a unique potential using the Fenchel-Legendre conjugacy in equation~\eqref{eq:legendre-conjugate} (also called c-transform in the OT literature). 
  We introduce for each $x \in \XC$ the conjugate of a pointwise potential $\varphi_\theta(\cdot, x)\colon \UC \to \YC$ as
  \begin{align}
  \label{eq:conditional-convex-conjugate}
     J_{\varphi_\theta(\cdot, x)}(u, y) &= u^T y - \varphi_\theta(u, x), \\
    \varphi^*_\theta(y, x) = J_{\varphi_\theta(\cdot, x)}\bigl(\check{u}_{\varphi_\theta(\cdot, x)}(y), y\bigr), &\quad \check{u}_{\varphi_\theta(\cdot, x)}(y) = \arg\max_{u \in \UC} J_{\varphi_\theta(\cdot, x)}(u,y).
  \label{eq:conditional-convex-conjugate2}
  \end{align}
  With these notations, the problem~\eqref{eq:dual_vqr} can be reformulated as
  the minimization of $\VC(\theta)$, defined as
  \begin{equation}
  \label{eq:dual-objective-1}
    \VC(\theta):= \EE_{(U,X) \sim F_U \otimes F_X}\left[\varphi_\theta(U,X)\right]+\EE_{(Y,X) \sim F_{YX}}\left[\varphi_\theta^*(Y,X)\right].
  \end{equation}
  Here, $F_U \otimes F_X$ denotes the product measure of $F_U$ and $F_X$, corresponding to independent sampling of $U \sim F_U$ and $X \sim F_X$.
  The optimal parameter, can be found by taking gradient steps of the dual objective with respect to $\theta$.
  The derivative goes through the loss and the Fenchel-Legendre conjugate is
  obtained by applying Danskin's theorem~\citep{danskin1967theory} and only requires the
  derivative of the potential 
  \[
    \nabla_\theta \VC(\theta)= \EE_{(U, X) \sim F_U \otimes F_X}[\nabla_{\theta} \varphi_\theta(U,X)] - 
    \EE_{(Y,X) \sim F_{YX}}[\nabla \varphi_\theta(u,X)\vert_{u=\check{u}_{\varphi_\theta(\cdot, X)}(Y)}].
  \]
  \begin{remark}
    Above we discuss the optimization of the dual potential $\varphi_\theta(\cdot, x)$ which is linked to $F_{U}$, with its conjugate $\varphi_\theta^*(\cdot, x)$ is linked to $F_{Y|X}(\cdot \mid X = x)$. But in practice, due to the symmetry of~\eqref{eq:dual_vqr}, one can instead use $\psi_\theta(\cdot, x)$. In our experiments we investigate both strategies.
  \end{remark}

\ecoparagraph{Neural Quantile Regression with PICNNs (C-NQR).}
  The convexity of $\varphi_\theta(\cdot, x)$ with respect to the first argument can be achieved~\citep{bunne2022supervised} by the usage of PICNNs~\citep{amos2017input}). However, the remaining challenge in solving the optimization problem in~\eqref{eq:dual-objective-1} arises from the fact that the conjugate $\varphi_\theta^*(\cdot, x)$ must be computed for each $x$ in the mini-batch. As a first strategy, we propose to do this exactly with an L-BFGS solver~\citep{liu1989limited}. The parameters of the PICNN potential $\varphi_\theta$ can be optimized using stochastic gradient descent (SGD); see \Cref{alg:neural_quantile_regression_training} in \Cref{sec:impl:train-loops} for implementation details. This approach is conceptually simple and uses existing optimization tools. However, it can be computationally intensive due to the repeated optimization required to compute the conjugates, especially for large mini-batches or high-dimensional data.

\ecoparagraph{Amortized Neural Quantile Regression (AC-NQR).}
  To reduce the computational cost of repeatedly solving the optimization problem~\eqref{eq:conditional-convex-conjugate2} to compute the conjugates, we propose an amortized optimization. The idea is to learn a predictor that approximates the solution of the conjugate problem, thus speeding up the inner optimization and training process. This strategy has been shown to be effective in the non-conditional case by~\citet{amos2023amortizingconvexconjugatesoptimal}.

  We extend this approach to the conditional case by introducing an amortization model $\tilde{u}_{\paramamor}(y, x)$ parameterized by $\paramamor$ that maps $(y, x)$ to a point that should ideally be close to the true solution $\check{u}_{\varphi_\theta(\cdot, x)}(y)$ in~\eqref{eq:conditional-convex-conjugate2}:
  \[
    \tilde{u}_{\paramamor}(y, x) \simeq \check{u}_{\varphi_\theta(\cdot, x)}(y).
  \]
  Note that different strategies have been proposed for the amortization model, but we will only focus on the one based on PICNNs such as in~\citep{makkuva2020optimal,korotin2019wasserstein}. The amortization model is trained jointly with the potential $\varphi_\theta$ by optimizing a quadratic loss that makes $\tilde{u}_{\paramamor}(y, x)$ to be close to $\check{u}_{\varphi_\theta(\cdot, x)}(y)$; see \Cref{alg:amoritzed_neural_quantile_regression_training} in \Cref{sec:impl:train-loops} for implementation details. This approach assumes that the amortization model evolves on a faster timescale than the potential $\varphi_\theta$, ensuring that its updates can track the slower dynamics of $\varphi_\theta$ during training, following the standard two-time-scale approximation~\citep{konda2004convergence,borkar2008stochastic}.

\ecoparagraph{Entropic regularized Neural Quantile Regression (EC-NQR).}
  Note that the two approaches discussed above requires the solution of a convex optimization problem to compute the exact conjugates, which becomes computationally intensive in high-dimensions. An alternative approach is to employ entropic regularization, enabling the use of stochastic gradient solvers~\citep{genevay2016stochastic}, which scale well but introduce bias that may distort the geometry of quantile maps~\citep{rosenberg2023fast}. Using a neural network to estimate the dual potentials  was considered  by~\citet{seguy2018large} for the non-conditional case and we propose to extend it to the conditional case for Neural VQR.

  This is done by adding an entropic regularization term to the primal OT problem, which smooths the problem and provides a closed-form solution for the conjugate (the argmax in~\eqref{eq:conditional-convex-conjugate2} becomes a softmax). This approach replaces the convex optimization required for conjugate optimization by an expectation that can be approximated with sampling; see \Cref{alg:entropic_semi_dual_training} in \Cref{sec:impl:train-loops} for implementation details. 
  More details on this approach and related works can be found in \Cref{supp:entropy_regularized_neural_optimal_transport} and \Cref{sec:SOTA} respectively.

\section{Conformal Prediction with OT Neural Maps}
\label{sec:conformal_ot_neural_maps}

  In this section, we demonstrate the use of our neural OT framework in constructing intrinsically adaptive confidence sets with CP. The key idea is to exploit multivariate quantile and rank maps learned by conditional neural OT as a building block for defining conformity scores and constructing calibrated prediction regions.  Let $(Y,X) \sim F_{YX}$ and $\alpha \in (0,1)$ and denote by $\widehat{Q}^{-1}_{Y \mid X}$ a proxy for the true associated vector rank function $Q^{-1}_{Y \mid X}$ as in \Cref{thm:carlier-existence}.

\ecoparagraph{Generalizing conformalized quantile regression.}
  In the univariate case, conformalized quantile regression (CQR; \citealp{romano2019conformalized}) replaces a nominal quantile with the empirical $(1-\alpha)$-quantile of residuals, ensuring distribution-free, finite-sample coverage. The same principle extends to the plug-in pullback set in~\eqref{eq:pullback_ball}. Define conformity scores
  \[
    S_i = \|\widehat{Q}^{-1}_{Y\mid X}(Y_i,X_i)\|, \qquad (Y_i, X_i) \in \DC_{\mathrm{cal}}.
  \]
  Let $S_{(1)} \leq \cdots \leq S_{(n)}$ denote the order statistics, set $k=\lceil (n+1)(1-\alpha)\rceil$, and $\rho_{1 - \alpha} = S_{(k)}$. The conformal set
  \[
    \hat{\CC}^\pullback_{\alpha}(x) = \{y\colon \widehat{Q}^{-1}_{Y\mid X}(y, x)\in \ball{0}{\rho_{1 - \alpha}}\}
  \]
  then guarantees $\PP_{(Y,X) \sim F_{YX}}\bigl(Y \in \hat{\CC}^\pullback_{\alpha}(X)\bigr) \geq 1 - \alpha$. We now show that this construction of confidence sets is optimal when the Jacobian of the inverse transport admits a radial structure.

  \begin{theorem}[Volume–optimality of pullback balls under radiality]
  \label{thm:radial_volume_optimality}
    Fix $x \in \XC$ and reference distribution $F_U(u) = \phi(\|u\|)$ for a strictly decreasing $\phi\colon [0,\infty) \to (0,\infty)$ on $\UC$, under the assumptions of \Cref{thm:carlier-existence}, let $Q_{Y|X}$ and $Q^{-1}_{Y|X}$ be the vector quantile and multivariate rank functions. Assume that there exists $j_x$ such that for all $y$ in the support of $F_{Y|X}$, it holds
    \begin{equation*}
      \det \closed{\nabla_y Q^{-1}_{Y|X}(y, x)} = j_x\open{\|Q^{-1}_{Y|X}(y, x)\|},
    \end{equation*}
    and the function $r \mapsto \phi(r) \, j_x(r)$ is strictly decreasing. Let $r_\alpha > 0$ be the unique radius satisfying $\mu(B_{r_\alpha}) = 1 - \alpha$, where $\mu$ is the law corresponding to $F_U$ and $B_r=\{u\colon \|u\| \le r\}$. Define the pullback ball $\CC^\pullback_\alpha(x)\coloneqq \ens{y\colon \|Q^{-1}_{Y|X}(y, x)\| \leq r_\alpha}$. Then, $\CC^\pullback_\alpha(x)$ minimizes Lebesgue volume among all sets with $x$-conditional coverage of at least $1 - \alpha$, i.e., for every measurable $A \subset \YC_x$ with $\PP\{Y \in A \mid X = x\} \ge 1 - \alpha$, $\vol\bigl(\CC^\pullback_\alpha(x)\bigr) \le \vol(A)$.
  \end{theorem}
  Equivalently, \Cref{thm:radial_volume_optimality} shows that $\CC^\pullback_\alpha(x)$ is the highest probability density (HPD) region for $Y \mid X = x$ at level $1-\alpha$. A noteworthy specialization, where the assumptions of \Cref{thm:radial_volume_optimality} are met, is the \emph{elliptical} case (including Gaussian) with $F_{Y\mid X}$ and $F_U$ belonging to the same elliptical family. We defer the proof and additional details to \Cref{sec:supp:conf_theory}.

\ecoparagraph{Re-ranked pullback sets.}
  This construction is effective only if the scores $S_i$ capture isotropic structure. Indeed, $\hat{\CC}^\pullback_{\alpha}(x)$ is the preimage of a centered Euclidean ball in $\UC$, implicitly assuming that the conditional distribution of $U = \widehat{Q}^{-1}_{Y \mid X}(Y,X)$ is radially symmetric. When $\widehat{Q}^{-1}_{Y \mid X}$ is misspecified, however, the ranks may be anisotropic, and Euclidean radii become unreliable. We note that the vector ranks $\{U_i = \widehat{Q}^{-1}_{Y \mid X}(Y_i, X_i)\}_{i=1}^n$ can themselves be interpreted as multivariate score functions and as such be combined with the OT-CP approach of \citet{thurin2025optimal}, which is designed to conformalize multivariate score functions. In particular, 
  let $\mathbf{R}\colon \UC \to \UC$ be a reranking approach, designed to correct deviations from reference distribution $F_U$. 
  Then, the conformalization step may be applied to the adjusted scores $\|\mathbf{R}(U_i)\|$, yielding a calibrated radius $\rho_{1-\alpha}^{\mathrm{uni}}$ and the prediction set
  \[
    \hat{\CC}^{\unipullback}_{\alpha}(x) = \bigl\{y\colon \mathbf{R}\bigl(\widehat{Q}^{-1}_{Y \mid X}(y, x)\bigr) \in \widehat{\QC}(1 - \alpha)\bigr\},
  \]
  where \(\widehat{\QC}(1 - \alpha)=\{u\colon \|\mathbf{R}(u)\| \le \rho_{1-\alpha}^{\mathrm{uni}}\}\). See additional implementation details in \Cref{sec:impl:conformal}

 \begin{remark}For completeness, we also consider a complementary construction that leverages the OT quantile and rank maps to estimate the conditional density via the change of variables formula. Using the estimated density as a conformal score, this approach yields valid regions and can capture disconnected geometry when $F_{Y\mid X=x}$ is multimodal, e.g. Gaussian mixture. We provide additional details and a brief discussion in \Cref{sec:supp:conf_theory}.
 \end{remark}

\section{Related Work}
\label{sec:related_work}

\ecoparagraph{Multivariate Quantiles.}
  Scalar quantile regression estimates conditional quantiles of $Y \in \RR$ given $X \in \RR^p$, with linear-in-features models fitted via the check loss~\citep{koenker1978regression,koenker2005quantile}. Multivariate extensions are harder due to the absence of a total order. Early notions include spatial quantiles~\citep{chaudhuri1996geometric} and depth-based quantiles~\citep{hallin2021distribution}, but these lack transport-map properties. A measure-transportation perspective defines multivariate quantiles as OT maps from a reference distribution, yielding center-outward ranks and quantile regions with strong properties~\citep{chernozhukov2017monge,hallin2021distribution,hallin2024multivariate,del2025nonparametric}. The conditional vector quantile function (CVQF) of \citet{carlier2016vector} uses affine-in-$X$ models estimated by variational OT~\citep{carlier2017vector}, with extensions to nonlinear embeddings~\citep{rosenberg2023fast}, continuous VQR~\citep{tallini2023continuous}, and nonparametric ranks~\citep{del2025nonparametric}. Scalable solvers rely on entropic regularization~\citep{carlier2022vector}; but to the best of our knowledge have never been scaled with Neural OT as we propose here.

\ecoparagraph{Neural Optimal Transport.}  
	High-dimensional OT is challenging due to the nonlinear dual formulation. One approach employs entropic regularization, enabling Sinkhorn iterations and stochastic gradient solvers~\citep{cuturi2013sinkhorn,genevay2016stochastic,seguy2018large,carlier2022vector}, which scale well but introduce bias that may distort the geometry of quantile maps~\citep{rosenberg2023fast}. A second approach parameterizes convex potentials with input-convex neural networks (ICNNs; \citealp{amos2017input,makkuva2020optimal,amos2023amortizingconvexconjugatesoptimal}), ensuring monotonicity and invertibility of the learned map. Conditional potentials (and Monge maps) have been proposed in \cite{bunne2022supervised} but are learned in a supervised way (from examples of conditioning and target distributions) and never from a unique joint sampling using the framework of \citet{carlier2017vector} as proposed in our work.

\ecoparagraph{Multivariate Conformal Prediction.}
	Conformal prediction (CP) constructs distribution-free predictive sets with coverage guarantees. In the scalar case, conformalized quantile regression (CQR; \citealp{shafer2008tutorial, romano2019conformalized,angelopoulos2023conformal}) adjusts quantile estimates to achieve valid intervals. For multivariate responses, naive coordinatewise CP yields conservative rectangles; scalarized scores via norms or maxima produce balls or boxes, but remain restrictive. Structured approaches include deep generative embeddings~\citep{feldman2023calibrated} and copula calibrations~\citep{messoudi2021copula}. \citet{dheur2025unified} propose conformity scores based on generative models or aggregated $p$-values. 

  Very recently, the use of OT-based ranks and quantiles has been exploited in conformal prediction. In two concurrent works, \citet{thurin2025optimal} define conformity scores from discrete OT ranks, while \citet{klein2025multivariate} leverage the same construction albeit with entropy regularized discrete OT. By construction, these two approaches are not adaptive, i.e. the size of the conformal set does not depend on $X$. Nonetheless, \citet{thurin2025optimal} propose an adaptive variant based on conditional  with k nearest neighbors. Our direct learning of neural VQR does not depend on conditional density estimation and should perform better in high dimensionality settings.

\section{Numerical Experiments}
\label{sec:numerical_experiments}

\subsection{Neural Optimal Transport}
  To evaluate the generative performance of our models, we conduct extensive experiments. Whenever a ground-truth operator is required, we parametrize the datasets using a convex potential function, see \Cref{sec:supp:optimal_transport_experiments_datasets} for details. EC-NQR, C-NQR$_U$, C-NQR$_Y$, AC-NQR$_U$, AC-NQR$_Y$ are the methods described in \Cref{sec:conformal_ot_neural_maps}. We measure the generative performance against FN-VQR \citep{rosenberg2023fast}, VQR \citep{carlier2017vector} and CPF \citep{huang2021convex}.
  
\label{sec:optimal_transport_experiments}

\ecoparagraph{Metrics.}
  We employ three categories of metrics: (i) Wasserstein-2 (W2) and Sliced Wasserstein-2 (S-W2) distances; (ii) Kernel Density Estimate $\ell_{1}$ distance (KDE-L1) and Kernel Density Estimate Kullback–Leibler divergence (KDE-KL); and (iii) Percentage of Unexplained Variance (L2-UV) \cite{korotin2021neural}. Metrics in (i) and (ii) quantify the fidelity of the learned distribution to the target density, while (iii) assesses the extent to which the ground-truth quantile is recovered. Additional implementation details are provided in \Cref{sec:supp:optimal_transport_experiments_datasets}.

  \begin{figure}[t!]
    \centering
    \includegraphics[width=0.9\linewidth]{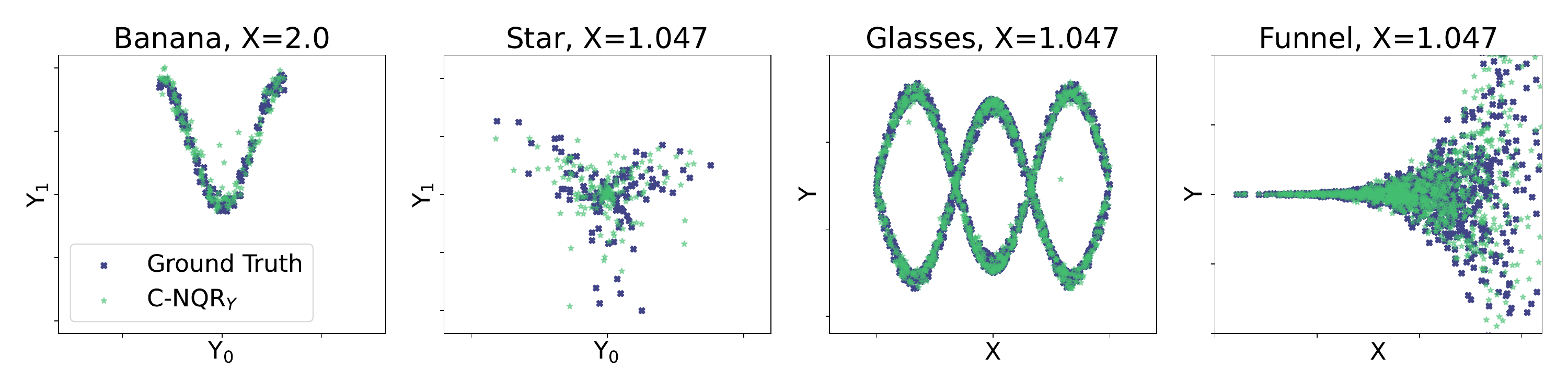}
    \vspace{-20pt}
    \caption{Example of points sampled from reference distribution of all the datasets we study and points sampled from approximation constructed by C-NQR$_U$ method.}
    \label{fig:sampled_points}
  \end{figure}

\ecoparagraph{Datasets.}
  We evaluate on three synthetic datasets originally introduced in the discrete setting of conditional quantile regression~\citep{rosenberg2023fast}:  
  \textit{Banana}, a parabola-shaped distribution whose curvature varies with a latent random variable;  
  \textit{Star}, a three-pointed star whose orientation is governed by a latent variable;  
  and \textit{Glasses}, a bimodal distribution with sinusoidally shifting modes.  
  We denote convex-potential counterparts as \textit{Convex Banana}, \textit{Convex Star}, and \textit{Convex Glasses}.

  Lastly, we consider Neal's Funnel~\citep{neal2003slice} \Cref{fig:sampled_points}. We extend this benchmark to higher dimensions by sampling $n$ independent samples from the distributions.
  
  \begin{table}[t!]
    \centering
    \resizebox{\textwidth}{!}{%
    \begin{tabular}{c|cccccccc}
      Dataset & EC-NQR & C-NQR$_U$ & C-NQR$_Y$ & AC-NQR$_U$ & AC-NQR$_Y$ & CPF & FN-VQR & VQR \\
      \midrule
      \textit{Star} & $0.197$ & $\underline{0.184}$ & $0.184$ & $\underline{\mathbf{0.182}}$ & $0.197$ & $0.247$ & $0.271$ & $0.270$ \\
      \textit{Glasses} & $\underline{\mathbf{0.748}}$ & $0.785$ & $0.812$ & $\underline{0.771}$ & $0.810$ & $1.687$ & $2.017$ & $1.964$ \\
      \textit{Banana} & $0.111$ & $0.072$ & $0.073$ & $0.073$ & $\underline{0.072}$ & $\underline{\mathbf{0.069}}$ & $0.398$ & $0.389$ \\
      \textit{Convex Star} & $0.200$ & $\underline{\mathbf{0.182}}$ & $0.184$ & $\underline{0.182}$ & $0.191$ & $0.191$ & $0.262$ & $0.261$ \\
      \textit{Convex Glasses} & $\underline{\mathbf{0.650}}$ & $\underline{0.656}$ & $0.668$ & $0.657$ & $0.689$ & $0.760$ & $1.954$ & $1.961$ \\
      \textit{Convex Banana} & $0.103$ & $0.101$ & $0.071$ & $\underline{0.070}$ & $0.070$ & $\underline{\mathbf{0.069}}$ & $0.397$ & $0.392$ \\
      \midrule
      Training time & 10.99 sec. & 15.08 sec. & 15.09 sec. & 8.89 sec. & 12.63 sec. & - & - & - \\
      Inference time & 1.71 sec. &  1.21 sec. & 1.76 sec. & 1.12 sec. & 1.34 sec. & - & - & - \\
      \bottomrule
    \end{tabular}%
    }
    \vspace{-10pt}
    \caption{S-W2 between ground truth and empirical distributions. We provide training time per epoch that is averaged over all the datasets and average inference time for computing c-transform inverse of 8192 elements.}
  \label{table:wasserstein_synthetic}
  \end{table}

\ecoparagraph{Results.}
 \Cref{table:wasserstein_synthetic} reports the median S-W2 metric across datasets. We find distance between $Y \mid X$ and $U \mid X$ to be most indicative of overall performance. We additionally report training and inference times: training time is reported as the median per-epoch duration across datasets, while inference time is the median cost of computing the $c$-transform on a batch of size $8192$ for a fully trained model. Further evaluation metrics with error bars are provided in the \Cref{sec:supp:optimal_transport_experiments_metrics}. We denote by C-NQR${_U}$ and AC-NQR${_U}$ the models estimating $\varphi(u,x)$, and by C-NQR${_Y}$ and AC-NQR${_Y}$ the models estimating $\psi(y,x)$; see equation~\eqref{eq:dual_vqr}.

  To evaluate scalability, \Cref{fig:funnel} reports the S-W2 metric on Neal's Funnel as the dimension of the target distribution increases from $2$ to $16$.  

  Finally we evaluate the ability of our method to recover the underlying generative structure. We report L2-UV metric in \Cref{table:unexplained_variance} evaluated on \textit{Convex Banana}, \textit{Convex Star} and \textit{Convex Glasses}.
  
  \begin{table}[t!]
    \centering
    \resizebox{\textwidth}{!}{%
    \begin{tabular}{cccccccc}
      Function & Dataset & \multicolumn{1}{c}{EC-NQR} & \multicolumn{1}{c}{C-NQR$_U$} & \multicolumn{1}{c}{C-NQR$_Y$} & \multicolumn{1}{c}{AC-NQR$_U$} & \multicolumn{1}{c}{AC-NQR$_Y$} & \multicolumn{1}{c}{CPF} \\
      \midrule
      \multirow{3}{*}{$Q^{-1}_{Y\mid X}$}
      & \textit{Convex Star} & $1.331$ & $\underline{\mathbf{0.211}}$ & $0.286$ & $\underline{0.264}$ & $0.425$ & $0.447$ \\
      & \textit{Convex Glasses} & $0.348$ & $0.332$ & $\underline{\mathbf{0.068}}$ & $0.203$ & $\underline{0.109}$ & $2.268$ \\
      & \textit{Convex Banana} & $3.942$ & $3.784$ & $0.212$ & $\underline{\mathbf{0.106}}$ & $\underline{0.206}$ & $9.479$ \\
      \midrule
      \multirow{3}{*}{$Q_{U\mid X}$} & Convex Star & $2.746$ & $0.360$ & $\underline{0.351}$ & $0.393$ & $0.525$ & $\underline{\mathbf{0.267}}$ \\
      & \textit{Convex Glasses} & $\underline{0.678}$ & $\underline{\mathbf{0.535}}$ & $0.732$ & $0.985$ & $1.096$ & $1.726$ \\
      & \textit{Convex Banana} & $9.400$ & $7.665$ & $0.660$ & $\underline{\mathbf{0.545}}$ & $\underline{0.569}$ & $16.537$ \\ 
      \bottomrule
    \end{tabular}%
    }
    \vspace{-10pt}
    \caption{L2-UV of the true quantile function measured on generative processes parameterized by convex potential networks.}
  \label{table:unexplained_variance}
  \end{table}

  \begin{figure}[t!]
    \centering
    \includegraphics[width=\linewidth]{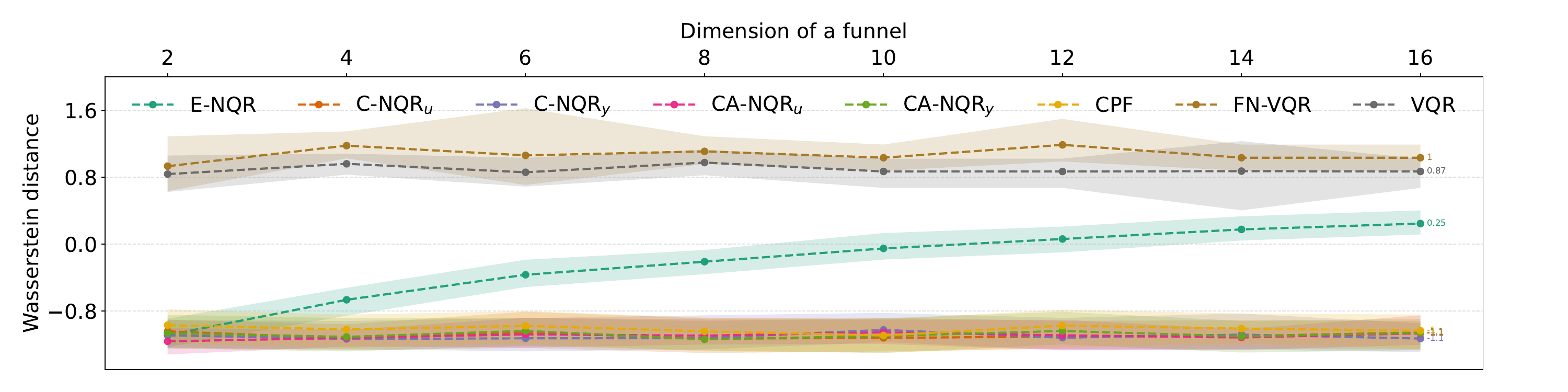}
    \vspace{-25pt}
    \caption{S-W2 on  Neal's funnel distribution. We scale the dimension of a funnel from 2 to 16.}
  \label{fig:funnel}
  \end{figure}


\subsection{Conformal Prediction Experiments}
\label{sec:conformal_experiments}
  We further evaluate conformal prediction by constructing prediction sets on real-world datasets using the methods described in \Cref{sec:conformal_ot_neural_maps}. Extended results are presented in Appendix~\ref{sec:supp:experiment_details}.

\ecoparagraph{Methods.}
  We use AC-NQR$_U$ as the base model to implement our two conformal methods: PB($\hat{\CC}^\pullback$) and RPB ($\hat{\CC}^\unipullback$). In addition to fitting our vector quantile regression model directly on $y$, we also fit both methods on signed residuals $s=y-\hat{f}(x)$, where $\hat{f}$ is a Random Forest regressor fit on $25\%$ of the training data (PBS and RPBS in the plots). We consider OT-CP and OT-CP+~\citep{thurin2025optimal}, as well as the local Ellipsoid method~\citep{messoudi2022ellipsoidal} for comparison.

\ecoparagraph{Metrics.}
  We evaluate performance using three metrics: (i) marginal coverage, (ii) worst-slab coverage~\citep{cauchois2021knowing}, and (iii) average prediction set volume.

\ecoparagraph{Datasets.}
  We evaluate on standard multi-target regression benchmarks used in previous work on uncertainty estimation~\citep{plassierrectifying, dheur2025unified}: \texttt{scm20d}, \texttt{sgemm}, \texttt{blog}, and \texttt{bio}. For the single-target datasets \texttt{blog} and \texttt{bio}, we follow~\cite{feldman2023calibrated} and add one of the features as a second output. The resulting response dimensions are 16, 4, 2 and 2, respectively. We use preprocessing procedure of~\citep{grinsztajn2022tree}.
  
  \begin{figure}
    \centering    
    \includegraphics[width=\textwidth]{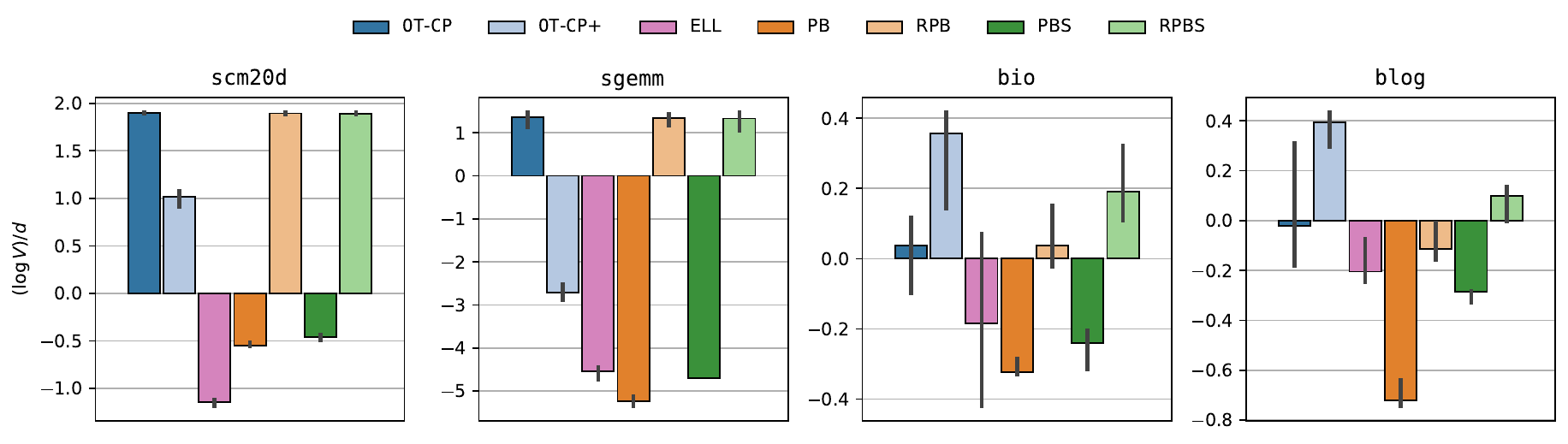}
    \vspace{-25pt}
    \caption{Log-volume of the prediction sets, normalized by $d_y$, of the resulting prediction sets for different methods. Results averaged over 10 independent data splits. Nominal miscoverage level $\alpha=0.1$}
    \label{fig:conf_volumes}
  \end{figure}

  \begin{figure}
    \centering    \includegraphics[width=\textwidth]{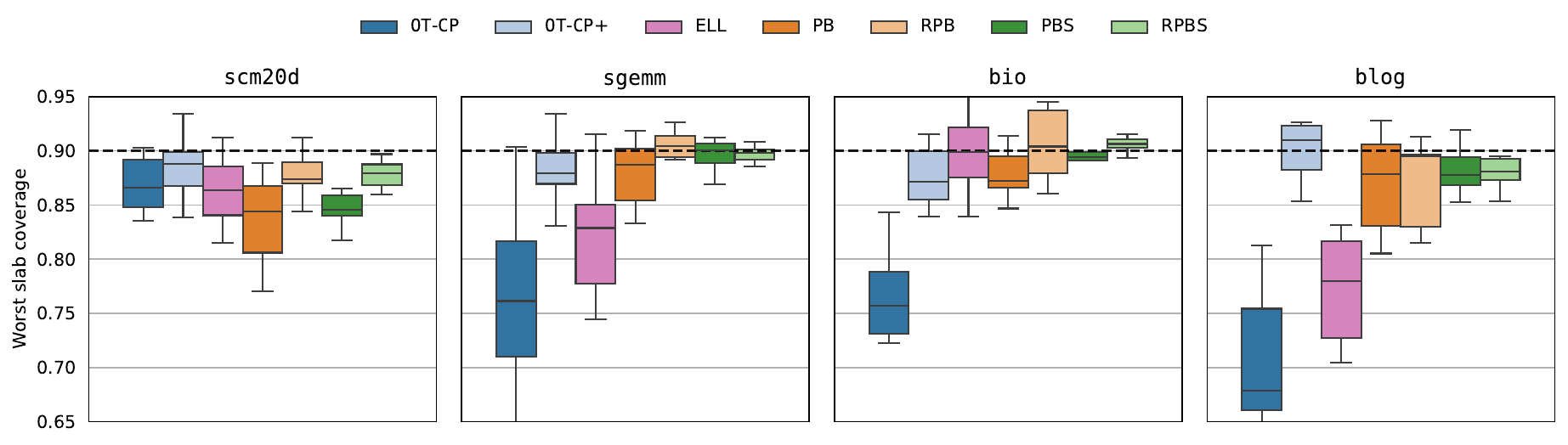}
    \vspace{-25pt}
    \caption{Worst-slab coverage for different methods. Results averaged over 10 independent data splits. Nominal miscoverage level $\alpha=0.1$}
    \label{fig:conf_volumes}
  \end{figure}
\ecoparagraph{Discussion.} PB and PBS provide competitive conditional coverage and smallest volume at the same time on three out of four datasets. The re-ranking step of RPB and RPBS allows to achieve a slightly sharper conditional coverage, but the increase in prediction sets volume make it a questionable trade-off. Overall, it shows that for our conditional quantile regression models the split conformal calibration is enough. Our methods provide a scalable training enable building competitive conformal predictors.




\section{Conclusion}
  We introduced a framework for multivariate conformal prediction based on convex potentials and optimal transport. Our approach leverages neural quantile regression with input convex neural network parameterization to construct valid and efficient prediction sets. Through experiments on synthetic benchmarks and real-world multi-target regression datasets, we demonstrated strong performance in terms of coverage and set size, while maintaining scalability in higher dimensions. Comparisons with existing baselines further highlight the robustness and flexibility of our method. Future work includes extending the framework to broader classes of generative models and exploring tighter efficiency guarantees in high-dimensional regimes.

\clearpage
\newpage

\section*{Usage of Large Language Models (LLMs)}
LLMs were used as a general-purpose assistive tool during the preparation of this paper. Their usage fell into two categories: (i) for writing assistance, they helped improve clarity and readability of certain passages through language refinement and (ii) for coding assistance, where they provided support with code completion and debugging. LLMs were not used for research ideation, experimental design, theoretical development, or analysis of results. All substantive contributions, including the conception of ideas, methodology, and experiments, were made by the authors.

\section*{Reproducibility Statement}
  We provide the full code to reproduce our experiments as supplementary material and will release it publicly upon acceptance.
  All experiments were conducted on publicly available datasets or datasets we created ourselves, which will be released alongside the code.
  We ran experiments with multiple seeds, if applicable, and report summary statistics.

\bibliographystyle{iclr2026_conference}
\bibliography{references.bib}

\clearpage
\newpage

\appendix

\section{Extended State of the Art}
\label{sec:SOTA}

\paragraph{From scalar to vector quantiles.} 
  Classical quantile regression (QR) estimates conditional quantiles of a scalar response $Y \in \RR$ given features $X \in \RR^p$, providing a flexible alternative to least squares for modeling heterogeneous effects~\citep{koenker1978regression,koenker2005quantile}. For a quantile level $u \in (0,1)$ and feature map $\varphi(x)$, a standard linear QR model assumes $Q_{Y\mid X}(u \mid x)=\beta(u)^\top \varphi(x)$, with $\beta(u)$ obtained by minimizing the check-loss. While univariate QR theory is well-developed, extending these notions to a multivariate response $Y \in \RR^d$ is challenging due to the lack of a natural total order on $\RR^d$. Many generalizations have been proposed, including \emph{directional} or \emph{projection} quantiles (reducing to scalar quantiles along particular directions) and \emph{geometric} or \emph{spatial} quantiles~\citep[e.g.][]{chaudhuri1996geometric}, as well as definitions based on statistical depth (e.g. Tukey's halfspace depth) that yield central regions viewed as multivariate ``quantiles.'' However, these early notions only partially extend scalar quantile properties and generally do not yield a unique quantile \emph{mapping} for $Y$. A recent breakthrough comes from the \emph{measure transportation} perspective, which defines multivariate quantiles as the optimal transport map pushing a reference distribution (usually the spherical uniform, or uniform on the unit hypercube) onto the distribution of $Y$. This approach—rooted in Brenier's theorem on monotone optimal transport maps~\citep{brenier1991polar}—yields well-behaved center-outward distribution and quantile functions that assign each point in $\RR^d$ a multivariate rank and sign with distribution-free properties. The resulting quantile regions are nested, have correct probability contents, and enjoy equivariance properties generalizing the one-dimensional case. These concepts, introduced by~\citet{chernozhukov2017monge} and further developed by~\citet{hallin2021distribution}, provide a rigorous multivariate analog of the quantile function; see~\citep{hallin2017multiple} for a survey of earlier definitions. Recent work continues to refine this framework: \citet{hallin2024multivariate} compare geometric vs. transport-based contours, and nonparametric multiple-output quantile regression methods based on center-outward ranks have been proposed~\citep{del2025nonparametric}.  

\paragraph{Vector quantile regression (VQR).}  
  Building on optimal transport ideas, \citet{carlier2016vector} introduced the \emph{conditional vector quantile function} (CVQF) $Q_{Y\mid X}(u,x)$ for $Y \in \RR^d$. This is defined as a (a.e.) \emph{monotone} map in $u$ — specifically, the gradient of a convex function in the $u$ argument — such that for each fixed $x$, $Q_{Y\mid X}(\cdot, x)$ pushes the uniform distribution on $[0,1]^d$ forward to the conditional distribution $Y\mid X=x$. In analogy to the scalar case, one can represent $Y$ as $Y = Q_{Y\mid X}(U, X)$ with $U \sim \text{Unif}([0,1]^d)$ independent of $X$. This generalizes the scalar quantile relationship $Y = Q_{Y\mid X}(U,X)$ for $U \sim \text{Unif}(0,1)$, providing a powerful characterization of the conditional law of $Y$ by a deterministic map on the unit hypercube. In practice, VQR imposes a parametric form on the CVQF; for example, the original proposal assumes an affine structure $Q_{Y\mid X}(u,x) = \alpha(u) + B(u)^\top x$ (with $\alpha(u) \in \RR^d$ and $B(u) \in \RR^{d\times p}$) and estimates these functions by solving a large-scale optimal transport problem under empirical data constraints. The solution can be found via a convex dual formulation analogous to Koenker's linear program, ensuring the fitted $Q_{Y|X}$ is monotone in $u$ (i.e. cyclically monotonic)~\citep{carlier2016vector,carlier2017vector}. This yielded the first notion of ``quantile regression for vectors,'' including strong theoretical guarantees on consistency and uniqueness under appropriate conditions. Since then, a number of extensions have been proposed: \citet{rosenberg2023fast} introduce a fast nonlinear VQR model (e.g. using kernel or neural network features) while preserving monotonicity, \citet{tallini2023continuous} develop a continuous VQR formulation that treats $u$ in a continuum (rather than on a finite grid of quantile levels), and fully nonparametric approaches based on center-outward quantile functions have appeared~\citep{del2025nonparametric}. Each of these methods seeks to balance flexibility and computational tractability while maintaining the defining property that $u \mapsto Q_{Y|X}(u,x)$ is a gradient map (hence invertible and order-preserving in the multivariate sense).  

\paragraph{Computation.}  
  Implementing VQR at scale poses significant challenges. The initial algorithms of~\citet{carlier2016vector} and~\citet{carlier2017vector} relied on discretizing the unit hypercube $[0,1]^d$ (for a set of representative $u$ values) and solving a large linear program, which becomes computationally expensive as $d$ or the number of quantile levels grows. Two recent strategies have substantially improved the scalability of VQR. First, \citet{carlier2022vector} propose an \emph{entropic regularization} of the OT problem, which smooths the objective and leads to a differentiable dual formulation. By applying Sinkhorn-type iterations or gradient-based optimization on the regularized dual, one can efficiently approximate the CVQF without solving a huge LP, even for continuous $u$ spaces. This regularized VQR approach yields an accuracy–computational cost trade-off controlled by the entropy penalty, and it has demonstrated orders-of-magnitude speedups on moderate-dimensional problems. 

  The second approach uses deep learning to represent the convex potential of the CVQF: \citet{makkuva2020optimal} propose to parameterize $Q_{Y|X}(u,x)$ as $\nabla_u \psi(u,x)$ where $\psi$ is an input-convex neural network in $u$. By training $\psi$ on data (using a suitable loss derived from the OT characterization), one obtains a VQR model that can handle high-dimensional $X$ and $Y$ and large sample sizes. This method, part of a broader trend of using neural networks for OT map estimation, sidesteps explicit discretization by leveraging automatic differentiation to enforce convexity in $u$. Both the entropic-OT and ICNN-based approaches have made it feasible to learn multivariate quantile functions on modern datasets, a task once thought impractical. For additional background on scalable optimal transport techniques that underlie these advances, see~\citep{peyre2019computational}.  

\paragraph{Conformal prediction.}  
  Conformal prediction (CP) provides distribution-free predictive uncertainty sets with finite-sample coverage guarantees. In the scalar $Y$ case, it is common to combine quantile regression with conformal calibration. For example, conformalized quantile regression (CQR) uses holdout data to adjust the initially estimated interval $[\hat{Q}_{Y|X}(\alpha/2 \mid x), \hat{Q}_{Y|X}(1-\alpha/2 \mid x)]$ so that it achieves the target coverage $1-\alpha$ marginally. CQR and related methods yield prediction intervals that are adaptive (varying with $x$) while retaining rigorous coverage guarantees~\citep{romano2019conformalized,angelopoulos2023conformal}. However, extending CP to multivariate outputs has proven more complex. Naively applying conformal methods to each component of $Y=(Y_1,\dots,Y_d)$ and taking a Cartesian product of marginal intervals yields a rectangular prediction region that is valid but often overly conservative (covering significantly more than $1-\alpha$ of the probability). More refined strategies have been proposed to account for dependence between coordinates. One line of work defines a scalar nonconformity score from the multi-output residual, for instance using a norm $|Y_{\text{pred}}-Y_{\text{true}}|$ or the maximum deviation across coordinates; this yields prediction balls or boxes aligned to the chosen norm. While simple, such choices typically lead to symmetric or axis-aligned regions that may be suboptimal in shape and volume. For example, the PCP method of~\citet{wang2023probabilistic} leverages an implicit generative model to draw random samples from $Y \mid X=x$ and constructs the prediction set as a union of Euclidean balls (of a fixed radius) centered at those samples. This approach guarantees marginal coverage and can improve sharpness over naive intervals, but using a global radius for all $x$ can lead to over-coverage in low-variability regions and under-coverage in high-variability regions. Alternatively, some works shape the prediction set as an ellipsoid by incorporating covariance structure: e.g. using a single global covariance estimate~\citep{johnstone2021conformal} or a local covariance around $x$~\citep{messoudi2022ellipsoidal} to define a Mahalanobis-distance conformity score. Such ellipsoidal regions capture linear correlations in $Y$ and are typically smaller than axis-aligned boxes, but they still assume an (approximately) elliptical and unimodal error distribution, which may be inappropriate for complex multimodal targets.

  Another class of methods seeks to learn a joint representation or dependency model for $Y$ before applying conformal. For example, \citet{feldman2023calibrated} train a deep generative model to embed $Y$ into a lower-dimensional (ideally unimodal) latent space and perform conformal quantile regression in that space, producing flexible regions when mapped back to $\RR^d$. Similarly, \citet{messoudi2021copula} and subsequent works leverage copula transformations: they calibrate marginal predictive intervals at miscoverage levels chosen to optimize the volume of the resulting joint region, effectively shaping the prediction set according to the dependence structure of $Y$. In particular, \citet{zhang2023improved} extend copula-based conformal prediction by allowing different significance levels for each output dimension and directly optimizing the hyperrectangle volume under the coverage constraint. \citet{sun2022copula}  provide a theoretical analysis of such copula-shaped prediction sets, proving that the empirical copula approach achieves finite-sample validity under i.i.d. assumptions. These methods produce tighter joint regions than the naive Cartesian product by allocating miscoverage intelligently across coordinates, though they often rely on either simple parametric copulas or numerical search to balance the marginal intervals.

  Very recently, \citet{dheur2025unified} conducted a comprehensive study of multi-output conformal methods, proposing in particular two new families of conformity scores. One uses a generative model (e.g. an invertible normalizing flow) to transform $Y$ into a space where conventional CP can be applied coordinate-wise, and the other defines a computationally efficient scalar score by combining coordinate-wise conformal $p$-values (essentially summing their logarithms). Both approaches attain finite-sample marginal coverage and offer improvements in conditional coverage. Notably, a conceptually similar idea was introduced concurrently by~\citet{fang2025contra}, who also leverage normalizing flows to define nonconformity in the latent space. Their method (CONTRA) maps high-density regions in the latent space to complex but high-coverage regions in output space, yielding non-axis-aligned prediction sets that outperform standard hyperrectangles or ellipsoids. Despite these advances, none of the above techniques exploits the full geometric structure of multivariate quantiles or ranks. This gap has been filled by two concurrent works that integrate the measure-transport perspective into conformal inference.

  \citet{thurin2025optimal} introduce OT-CP, which uses the center-outward rank function of~\citet{hallin2021distribution} to define multivariate order statistics. In essence, they compute the “rank” of a test point $y$ among past observations in $\RR^d$ via the empirical center-outward distribution (obtained by optimal transport), and use the corresponding multivariate quantile level as the nonconformity score. This yields a prediction region for a new $X=x$ by including all $y$ whose center-outward rank is above a certain quantile (determined by the calibration set)—intuitively, the set of points that lie among the $(1-\alpha)$ fraction most central (least outlying) under the conditional distribution of $Y\mid X=x$. Independently, \citet{klein2025multivariate} develop a related approach that also relies on optimal transport to order multivariate outputs. They formalize the notion of distribution-free multivariate quantile regions and provide finite-sample coverage guarantees for both exact and approximate transport maps. These OT-based conformal methods leverage the geometry of Brenier maps (i.e. conditional Monge–Ampère transports) to construct flexible, data-dependent prediction sets in $\RR^d$ that adapt to the local distribution of $Y \mid X=x$. By exploiting the vector-quantile structure, they can achieve tighter coverage with complex (even non-convex) regions while still guaranteeing the rigorous coverage properties that make conformal prediction attractive. However, the use of optimal transport maps can be computationally expensive in high dimensions, and in practice one might need to trade off some statistical efficiency for tractability when estimating the transport.

  Finally, an alternative direction is to explicitly optimize prediction set volume subject to coverage, rather than relying on a fixed conformity score. \citet{braun2025minimum} propose an optimization-driven framework that learns minimum-volume covering sets for multivariate regression. In their approach, the predictive model is trained jointly with a parametric prediction set (for example, an adaptive norm-ball whose radius may vary with $x$) to minimize the volume of the set while enforcing coverage on the training data via a surrogate loss. This procedure effectively learns the shape of the prediction region that best captures a specified proportion of the data. By conformalizing the learned region (i.e. slightly expanding it to guarantee $1-\alpha$ coverage on a holdout set), the method yields valid prediction sets that are much tighter than those from standard split-conformal methods. Such approaches highlight an exciting trend of combining machine learning and conformal inference: rather than treating the prediction algorithm as a black box, one can optimize the model and its uncertainty quantification in tandem to achieve improved efficiency (smaller, more informative prediction sets) without sacrificing the finite-sample guarantees of CP.

\section{Entropy-Regularized Neural VQR}
\label{supp:entropy_regularized_neural_optimal_transport}

  Let $\XC,\YC,\UC$ be Polish spaces with Borel $\sigma$–algebras, and let $m$ be the marginal law of $X$, $\nu(\rmd x,\rmd y)=m(\rmd x)\,\nu_z(\rmd y)$ the joint law of $(X,Y)$, and $\mu(\rmd x,\rmd u)=m(\rmd x)\,\bar{\mu}(\rmd u)$ the joint law of $(X,U)$ (where $\bar{\mu}$ is the marginal distribution of $U$). For $\varepsilon>0$, the entropic-regularized \emph{conditional} OT problem reads~\citep{carlier2022vector}
  \begin{equation}
  \label{eq:primal-entropic}
    \min_{\gamma \in \MC_+(\XC \times \YC \times \UC)}
    \Big\{-\int u^\top y \rmd\gamma + \varepsilon\,\mathrm{KL}\big(\gamma ~\Vert~ \bar{\mu} \otimes \nu\big)\Big\}
    \quad\text{s.t.}\quad
    \Pi_{X,Y}\#\gamma=\nu, \Pi_{X,U}\#\gamma=\mu.
  \end{equation}
  This is a strictly convex problem with linear marginal constraints; $\mathrm{KL}$ denotes the Kullback–Leibler divergence. \eqref{eq:primal-entropic} specializes the standard entropic OT to the conditional setting by constraining the two $(X,\cdot)$ marginals of $\gamma$.

\ecoparagraph{Dual formulation via Fenchel–Rockafellar.}
  We introduce the dual potentials $\psi\colon \XC \times \YC \to \RR$ and $\varphi\colon \XC \times \UC \to \RR$. A direct application of Fenchel–Rockafellar duality yields the (unconstrained) dual
  \begin{multline}
  \label{eq:dual-entropic}
    \inf_{\psi,\varphi}
    \underbrace{\int \psi(y, x)\,\nu(\rmd x,\rmd y)}_{\text{term for }\Pi_{X,Y}}
    +
    \underbrace{\int \varphi(u, x)\,\mu(\rmd x,\rmd u)}_{\text{term for }\Pi_{X,U}}
    \\ 
    +
    \varepsilon \int \int
    \exp\Big(\tfrac{u^\top y-\varphi(u, x)-\psi(y, x)}{\varepsilon}\Big)\,
    \nu(\rmd x,\rmd y) \bar{\mu}(\rmd u),
  \end{multline}
  with zero duality gap and attainment under mild assumptions (tightness and finite entropy)
  The inequality constraint of the unregularized dual is absorbed by the exponential term  barrier in~\eqref{eq:dual-entropic}. This could be solved using purely stochastic optimization with NN parameterization of the two dual potentials $\psi(y, x)$ and $\varphi(u, x)$ similarly to what was proposed in~\citep{genevay2016stochastic, seguy2018large}. But from a practical perspective the exponential in the loss is particularly hard to optimize with numerical stability problems. This is why \citet{genevay2016stochastic} proposed to remove one of the potentials using the smooth version of the $c$–conjugacy detailed below.

\ecoparagraph{KKT/first-order conditions: \emph{soft} $c$–conjugacy.}
  Assuming $\nu$ and $\bar{\mu}$ admit densities w.r.t. the Lebesgue measure and differentiating the dual objective in~\eqref{eq:dual-entropic} w.r.t.\ $\psi$ and $\varphi$ gives the optimality (stationarity) conditions
  \begin{align}
  \label{eq:soft-psi-of-phi}
    \psi_\varepsilon(y, x)
    &= \varepsilon \log \int
    \exp\Big(\tfrac{u^\top y-\varphi_\varepsilon(u, x)}{\varepsilon}\Big) \bar{\mu}(\rmd u),
    \\
    \label{eq:soft-phi-of-psi}
    \varphi_\varepsilon(u, x)
    &= \varepsilon \log \int
    \exp\Big(\tfrac{u^\top y-\psi_\varepsilon(y, x)}{\varepsilon}\Big) \nu_x(\rmd y),
  \end{align}
  which are the entropic (``soft'') $c$–transforms, i.e., log-partition functions of exponential families induced by the bilinear cost $c(u,y)=-u^\top y$.
  At $\varepsilon\downarrow 0$, the identities~\eqref{eq:soft-psi-of-phi}–\eqref{eq:soft-phi-of-psi} $\Gamma$–converge to the hard Fenchel conjugacy $\psi=\varphi^\star$, recovering the unregularized dual feasibility $\varphi(u, x)+\psi(y, x)\ge u^\top y$ with equality on the support of the optimal plan.

\ecoparagraph{Reduction to a single potential (semi-dual).}
  Eliminating $\psi$ in~\eqref{eq:dual-entropic} via~\eqref{eq:soft-psi-of-phi} yields an equivalent unconstrained problem in $\varphi$:
  \begin{equation}
  \label{eq:phi-only-objective}
    \UC_\varepsilon(\varphi)
    =
    \EE_{(X,U)\sim \mu}\big[\varphi(U, X)\big]
    +
    \EE_{(X,Y)\sim \nu}
    \bigg[\varepsilon \log \int
    \exp\Big(\tfrac{u^\top Y-\varphi(u, X)}{\varepsilon}\Big) \bar{\mu}(\rmd u)\bigg],
  \end{equation}
  which is precisely the regularized analogue of the conjugate-based loss in the unregularized case (log-sum-exp replaces the $\sup$). This problem is very interesting from an optimization perspective because now a unique dual potential needs to be optimized and the log-sum-exp can be implemented in  a much more stable way than the exponential in the dual \eqref{eq:dual-entropic}. But then the inner expectation in the right part of \eqref{eq:phi-only-objective} cannot be computed exactly, which we discuss next.

\ecoparagraph{Gibbs conditionals and gradients.}
  Define the Gibbs conditional density (a.k.a.\ Schrödinger bridge ``posterior'')
  \[
    \pi_\varphi(\rmd u\mid y, x)
    \propto
    \exp\Big(\tfrac{u^\top y-\varphi(u, x)}{\varepsilon}\Big) \bar{\mu}(\rmd u).
  \]
  As in the not regularized case, we parameterize the potential $\varphi_\epsilon$ with a neural network. We denote by $\theta$ the parameters (weights) of this network. Using the log-partition derivative identity, we get that $\nabla_\theta \, \UC_\varepsilon(\varphi_\theta)$ admits the ``positive minus negative phase'' form
  \begin{equation}
  \label{eq:gradient}
    \nabla_\theta \, \UC_\varepsilon(\varphi_\theta)
    =
    \EE_{(X,U)\sim \mu}\big[\nabla_\theta \varphi_\theta(X,U)\big]
    -
    \EE_{(X,Y)\sim \nu}\,\EE_{U\sim \pi_{\varphi_\theta}(\cdot\mid Y,X)}
    \big[\nabla_\theta \, \varphi_\theta(X,U)\big],
  \end{equation}
  obtained by differentiating the log-partition in~\eqref{eq:phi-only-objective}. In practice, the inner expectation is estimated by Monte Carlo with $U$ drawn either from $\pi_{\varphi_\theta}(\cdot\mid Y,X)$ or via importance sampling from $\bar{\mu}$ with the usual exponential weights.

\ecoparagraph{Quantile and rank maps under entropic regularization.}
  If $u\mapsto \varphi_\varepsilon(u, x)$ is (strongly) convex and smooth, the regularized analogues of the conditional vector quantile and rank are
  \begin{align}
  \label{eq:q-map-entropy}
    Q^{(\varepsilon)}_{Y\mid X}(u, x) &:= \nabla_u \varphi_\varepsilon(u, x),\\
    \label{eq:rank-entropy}
    \big(Q^{(\varepsilon)}_{Y\mid X}\big)^{-1}(y, x)
    &:= \nabla_y \psi_\varepsilon(y, x)
    = \EE_{U\sim \pi_{\varphi_\varepsilon}(\cdot\mid y, x)}[U],
  \end{align}
  where the last identity follows by differentiating~\eqref{eq:soft-psi-of-phi}. Equations~\eqref{eq:q-map-entropy}–\eqref{eq:rank-entropy} are the entropic counterparts of the unregularized identities and reduce to them as $\varepsilon\downarrow 0$.

\ecoparagraph{Limit $\varepsilon\downarrow 0$.}
  As $\varepsilon\to 0$, $\varepsilon\log \int \exp((\cdot)/\varepsilon)\to \sup(\cdot)$, so
  \[
    \UC_\varepsilon(\varphi) \xrightarrow[\varepsilon\downarrow 0]{}
    \EE_{\mu}[\varphi(X,U)] + \EE_{\nu}\big[\varphi^\star(X,Y)\big],
  \]
  recovering the unregularized loss with the hard Fenchel conjugate and the transition from the constrained dual (inequality) to the unconstrained conjugate form. In the same limit, $\pi_\varphi(\cdot\mid y, x)$ concentrates on the (possibly set-valued) argmax of $u\mapsto u^\top y-\varphi(u, x)$, and~\eqref{eq:q-map-entropy}–\eqref{eq:rank-entropy} converge to the OT maps of the unregularized problem.

\section{Conditional Convex Potential Flows}
\label{supp:conditional_convex_potential_flows}

\paragraph{Conditional (partially convex) construction.}
  Given covariates $x \in \XC$, we model the conditional transport by a \emph{partially} input–convex potential
  \[
    \varphi\colon \RR^d\times\XC \to \RR, \qquad u\mapsto \varphi(u;x)\ \text{convex (strongly convex) for each fixed }x,
  \]
  and define the \emph{conditional convex potential flow} (a.k.a. \emph{partially convex potential flow})
  \[
    Q_{Y\mid X}(u,x) := \nabla_{u}\varphi(u;x),\qquad U \sim f_U\ \leadsto\ Y \mid X=x\ \text{via}\ Y=Q_{Y\mid X}(U,x).
  \]
  Under absolute continuity of $f_{Y\mid X}(\cdot\mid x)$ (see \Cref{assum:absolute_continuity_cond_distr}), the conditional \emph{rank} map (inverse quantile) exists and is the gradient of the conjugate:
  \[
    Q_{Y\mid X}^{-1}(y,x) = \nabla_y \varphi^\star(y;x),
  \]
  and the two maps are inverses a.e.\ (in $u$ and $y$) for each $x$. For any $(y,x)$ such that the inverse is well defined.
  \begin{equation}\label{eq:cond-cov}
    f_{Y\mid X}(y\mid x) =  f_U\big(Q_{Y\mid X}^{-1}(y,x)\big) \det\Big[\nabla_y Q_{Y\mid X}^{-1}(y,x)\Big].
  \end{equation}
  Equivalently, writing $y=\nabla_u\varphi(u;x)$ with $u=Q_{Y\mid X}^{-1}(y,x)$,
  \[
    \log f_{Y\mid X}(y\mid x)\;=\;\log f_U(u)\;-\;\log\det\!\big[\nabla_{uu}^2\varphi(u;x)\big].
  \]
  Thus maximum likelihood amounts to estimating $\varphi$ so as to match the pullback $Q_{Y\mid X}^{-1}(Y,X)$ to the prior $f_U$, while penalizing the local volume change through the (log) Hessian determinant. In practice, the log-determinant and its gradients can be computed with Hessian–vector products, using stochastic Lanczos/trace estimators and conjugate-gradient solves, yielding unbiased $O(1)$–memory estimators that scale to high dimension.

\paragraph{Inversion and sampling.}
  For any $(y,x)$, inversion is a convex program:
  \[
    Q_{Y\mid X}^{-1}(y,x) = \arg\min_{u\in\mathbb{R}^d}\ \varphi(u;x)-y^\top u,
  \]
  whose optimality condition $\nabla_u\varphi(u;x)=y$ recovers the required $u$. This is precisely the evaluation of $\nabla_y\varphi^\star(y;x)$ and can be carried out with off-the-shelf smooth convex solvers; batched inversions reduce to minimizing summed potentials over independent inputs.

  Under mild regularity (convex support and densities), there exists a measurable conditional vector quantile $Q_{Y\mid X}$ that is the gradient (in $u$) of a convex potential and pushes $U$ to $Y\mid X=x$; the inverse rank is the gradient (in $y$) of the conjugate, and $Q_{Y\mid X}$ solves the $W_2$ OT problem conditionally on $x$. Hence the partially convex potential flow inherits both identifiability (a.e.\ uniqueness) and optimality properties in the conditional setting.

\paragraph{Parameterization.}
  We instantiate $\varphi(\cdot;x)$ with partially input–convex networks (e.g., PICNN/PISCNN) to guarantee convexity in $u$ while conditioning on $x$, and add a quadratic $\tfrac{\alpha}{2}\|u\|^2$ when strong convexity is desired. Universality of ICNNs in approximating convex functions then lifts to distributional universality of the induced conditional flows and convergence to the conditional OT maps.

\section{Numerical Implementation}
  This section details architectures, solvers, and training procedures for our neural vector quantile regression (VQR) models, both in the unregularized and entropic-regularized settings, together with the amortized conjugate prediction used to accelerate training. We emphasize implementation choices that preserve convexity/monotonicity and lead to stable gradients, and we provide concrete defaults for reproducibility.


\ecoparagraph{Notation recap.}
  We parameterize a \emph{conditional convex potential} $\varphi_\theta\colon \UC \times \XC \to \RR$ that is convex in $u\in\UC\subset\RR^{\dimy}$ for each fixed $x\in\XC$. The conditional vector quantile and rank maps are the gradients of $\varphi_\theta$ and its Fenchel conjugate $\varphi_\theta^\star$ (see Section~2): 
  \[
    Q_{Y\mid X}(u,x)=\nabla_u\varphi_\theta(u,x), 
    \qquad 
    Q^{-1}_{Y\mid X}(y,x)=\nabla_y\varphi_\theta^\star(y,x).
  \]
  The conjugate evaluation at $(y,x)$ solves $\widehat{u}_{\theta}(y,x)\in\arg\max_{u\in\UC}\{u^\top y-\varphi_\theta(u,x)\}$. By Danskin's theorem, gradients w.r.t.\ $\theta$ do \emph{not} require differentiating through $\widehat{u}_\theta$; only $\nabla_\theta \varphi_\theta$ at $u = \widehat{u}_\theta$ is needed.

\subsection{Partially  Input Convex Neural Networks (PICNN)}
\label{sec:supp:PICNN}
  We instantiate $\varphi_\theta$ as a \emph{Partially Input Convex Neural Networks} (PICNNs; \citealp{amos2017input}) that is convex in $u$ and conditions on $x$:
  \[
    (u,x)\ \longmapsto\ \varphi_\theta(u,x)=\mathrm{PICNN}(u,x;\theta),
  \]
  with layerwise updates
  \[
    \begin{aligned}
      c_{i+1}&=\tilde g_i(\tilde W_i c_i+\tilde b_i),\\
      z_{i+1}&=g_i\Big(W^{(z)}_i \big( z_i \circ \big[ W^{(zc)}_i c_i+b^{(z)}_i \big]_+\big) + W^{(u)}_i\big[ u \circ \big(W^{(uc)}_i c_i+b^{(u)}_i\big)\big]+\ W^{(c)}_i c_i+ b_i\Big), 
    \end{aligned}
  \]
  and output $\varphi_\theta(u,x)=z_K$. We initialize $c_0= x, z_0=0$. Here $\circ$ denotes the element-wise product. We enforce elementwise nonnegativity of $W^{(z)}_i$ and $[ \cdot ]_+$ via a Softplus reparameterization:
  \begin{align}
    W_i^{(z)} &= \log \Bigl(1+\exp\Bigl(\tilde{W}^{(z)}_{i}\Bigr)\Bigr), 
    \quad \tilde{W}^{(z)}_i \in \RR^{p\times k}, \\
    [W^{(zc)}_i c_i+b^{(z)}_i]_+ & = \log \Bigl(1+\exp\Bigl(W^{(zc)}_i c_i+b^{(z)}_i\Bigr)\Bigr).
  \end{align}
  We use convex, non-decreasing activations for $g_i,\tilde g_i$, which guarantees convexity in $u$ while retaining expressive power. We optionally add a quadratic term $\frac{\alpha}{2}\|u\|_2^2$ (trainable $\alpha\ge 0$) to obtain strong convexity, improving stability of the inner argmax~\citep[Proposition~2]{amos2017input}. We choose \texttt{Softplus} as non-linearity for $g_i$ and \texttt{ELU} as non-linearity for $c_i$. Following~\citet{huang2021convex} we utilize activation normalization \texttt{ActNorm} layers~\citep{kingma2018glow} before applying the $g_i$ non-linearity. Final architecture of one iterate hence becomes.
  \[
    \begin{aligned}
      c_{i+1}&=\text{ELU}(\tilde W_i c_i+\tilde b_i),\\
      z_{i+1}&=\text{Softplus}\Big( \text{ActNorm}\Big( W^{(z)}_i \big( z_i \circ \big[ W^{(zc)}_i c_i+b^{(z)}_i \big]_+ \big) \\
      &+ W^{(u)}_i\big( u \circ \big[W^{(uc)}_i c_i+b^{(u)}_i\big]\big) +W^{(c)}_i c_i + b_i\Big)\Big), 
    \end{aligned}
  \]

\ecoparagraph{Practical tips (PICNN).}
  \begin{enumerate}[label=(\roman*)]
    \item Normalize $u$ and $y$ scales (e.g.\ standardization) to ease optimization;

    \item We use weight decay on $\theta$ and (if enabled) a small ridge $\alpha$ to avoid flat directions;

    \item We clip gradients of $\varphi_\theta$ to bound the Lipschitz constant of $u\mapsto\nabla_u\varphi_\theta(u,x)$.
  \end{enumerate}

\subsection{Partially Input Strongly Convex Neural Network (PISCNN)}
\label{sec:supp:PISCNN}
  \[
    \mathrm{PISCNN}(u,x)=\mathrm{PICNN}(u,x)+\frac{\alpha}{2}\,\|u\|_2^2,
  \]
  which is strongly convex in $u$ and yields a \emph{strictly concave} inner objective $u\mapsto u^\top y-\varphi_\theta(u,x)$, ensuring a unique maximizer $\widehat{u}_\theta(y,x)$ and faster, more reliable inner solves. We treat $\alpha$ as positive scalar parametrized by $e^w$, where $w$ is a trainable parameter. In all our implementations, enabling $\alpha>0$ eliminated numerical non-uniqueness in the conjugate and reduced inner iterations.

\subsection{Computing the Conjugate: Inner Maximization}
\label{sec:impl:inner}
  Given $(y,x)$ and current $\theta$, we compute 
  \[
    \widehat{u}_\theta(y,x)\in\arg\max_{u\in\UC}\ J_\theta(u;y,x), \quad J_{\theta}(u; y, x):=u^\top y-\varphi_\theta(u,x).
  \]

\ecoparagraph{Gradient and Hessian.}
  $\nabla_u J_\theta(u;y,x)=y-\nabla_u\varphi_\theta(u,x)$ and $\nabla^2_{uu}J_\theta(u;y,x)=-\nabla^2_{uu}\varphi_\theta(u,x)\preceq 0$. With PISCNN, $\nabla^2_{uu}\varphi_\theta(u, x)\succeq \alpha \Id$ ensures strong concavity.

\ecoparagraph{Solver.}
  We minimize $-J_\theta$ with L-BFGS. For stability:
  \begin{enumerate}
    \item \textbf{Warm start.} We initialize the solver from amortized predictor $\tilde u_\vartheta(y,x)$ when available (see Section~\ref{sec:impl:amort}) or otherwise initialize it at $u\sim F_U$.

    \item \textbf{Domain handling.} If $\UC$ is a ball/hypercube, we project the solution after each step: $u\leftarrow\Pi_{\UC}(u)$.

    \item \textbf{Stopping.} Terminate when $\|\nabla_u J_\theta\|\le \varepsilon_{\text{norm}}$, $\|J_{\theta}(u_i; y, x) - J_{\theta}(u_{i+1}; y, x)\|\le \varepsilon_{\text{obj}}$ or after $K_{\max}$ steps (defaults: $\varepsilon_{\text{norm}}=10^{-7},\varepsilon_{\text{obj}}=10^{-7}, K_{\max}=1000$).
  \end{enumerate}

\subsection{Amortized conjugate prediction}
\label{sec:impl:amort}
  To avoid expensive inner solves at every iteration, we learn a differentiable predictor $\tilde u_\vartheta:\YC\times\XC\to\UC$ that approximates $\check{u}_\vartheta(y,x)$ and serves as a warm start for L-BFGS solver. We parametrize $\tilde u_\vartheta(y,x)$ as feed forward neural network with a residual skip connection to encourage identity at initialization
  \[
    \tilde u_\vartheta(y,x)=\mathrm{MLP}_\vartheta
    \Big(\begin{bmatrix} 
      y \\ 
      x 
    \end{bmatrix}\Big) + W_y y + b_y.
  \]

\ecoparagraph{Training losses.}
  Several loss functions have been explored in the literature. Objective-based losses~\citep{dam2019three, amos2023amortizingconvexconjugatesoptimal} optimize the network to predict the maximum of the conjugate by maximizing $\VC_{\text{obj}} = \EE_{(Y,X) \sim F_{YX}}[ J_{\theta}(\tilde{u}_{\paramamor};y, x)]$. Alternatively, one may enforce the first-order condition $\nabla_u \varphi_{\theta}(u, x)|_{u= \tilde{u}_{\paramamor}(y, x)} \approx y$ via the residual loss $\VC_{\text{res}}(\paramamor) = \EE_{(Y,X) \sim F_{YX}}[\lVert \nabla_u \varphi_{\theta}(u, x)|_{u= \tilde{u}_{\paramamor}(y, x)} - y \rVert^2_2]$. If the true conjugate $\check{u}_{\varphi_\theta(\cdot, x)}(y)$~\eqref{eq:conditional-convex-conjugate} is available, one can regress directly with $\VC_{\tilde{u}} = \EE_{(Y,X) \sim F_{YX}}[\lVert \tilde{u}_{\paramamor}(y, x) - \check{u}_{\varphi_\theta(\cdot, x)} (y)\rVert^2_2]$. In practice, we observe no major differences between these approaches and therefore adopt $\VC_{\tilde{u}}$ as our loss of choice (see \Cref{alg:amoritzed_neural_quantile_regression_training}).




\subsection{Entropic-regularized Semi-dual}
\label{sec:impl:entropic}
  When using the entropic semi-dual $U_\varepsilon(\varphi)$ (see \Cref{supp:entropy_regularized_neural_optimal_transport}), we replace the hard conjugate with a log-sum-exp:
  \[
    U_\varepsilon(\varphi_\theta)=\E_{(X,U)}[\varphi_\theta(U,X)] + \E_{(X,Y)}\left[\,\varepsilon\log\E_{U\sim F_U}\exp\left(\frac{U^\top Y-\varphi_\theta(U,X)}{\varepsilon}\right)\right].
  \]

\ecoparagraph{Monte Carlo and stability.}
  We approximate the inner expectation with $m$ i.i.d.\ samples $U_j\sim F_U$, using a numerically stable log-sum-exp with $64$-bit accumulation. We found $m\in[512, 1024]$ adequate on our benchmarks, and we re-sample the $U_j$ each iteration. In the $\varepsilon\downarrow 0$ limit, this recovers the unregularized loss. We intentionally set high amount of samples for dual objective estimation to avoid effects related to high bias of logsumexp estimator. 

\ecoparagraph{Gradients.} 
  The gradient has a positive-minus-negative phase form using the Gibbs weights (see \Cref{supp:entropy_regularized_neural_optimal_transport} and equation~\eqref{eq:gradient}), which we implement without storing the full batch $\times m$ tensor by streaming accumulation.

\subsection{Training Loops and Algorithms}
\label{sec:impl:train-loops}
  We describe three loops: (i) Neural Vector Quantile Regression without amortization \Cref{alg:neural_quantile_regression_training}, (ii) Amortized Vector Quantile Regression \Cref{alg:amoritzed_neural_quantile_regression_training}, and (iii) Entropic Semi-dual \Cref{alg:entropic_semi_dual_training}. All use \texttt{AdamW} (initial LR of $10^{-2}$, weight decay $10^{-4}$) with cosine annealing (LR decaying to $0$), batch size $1024$, and gradient clipping at $10$. We sample $U\sim F_U$ as standard Gaussian unless otherwise noted. See \Cref{sec:supp:optimal_transport_experiments_datasets} for dataset-specific details. We use warm restarts for amortized network, restarting the learning rate to $10^{-2}$ each 10 epochs.

  \begin{algorithm}[h!]
    \caption{Neural Vector Quantile Regression Training (C-NQR)}
    \label{alg:neural_quantile_regression_training}
    \begin{algorithmic}[1]
      \State \textbf{Input:} dataset $\{(x_i, y_i)\}_{i = 1}^n$, PICNN $\varphi_\theta\colon \UC \times \XC \to \RR$
      \State Sample mini-batch $\BC \subset \{1, \dots, n\}$
      \State Initialize $\VC_\varphi \gets 0$
      \For{each $i \in \BC$}
        \State $\check{u}_i \gets \arg\max_{u \in \UC} J_{\varphi_\theta(\cdot, x_i)}(u, y_i)$ \Comment{Run L-BFGS for each $y_i$ starting at $u = 0$}
        \State $\widehat{\psi}_i(\theta) \gets J_{\varphi_\theta(\cdot, x_i)}(\check{u}_i, y_i)$
        \State Sample $u_i \sim \NC(0,I_d)$
        \State $\widehat{\varphi}_i(\theta) \gets \varphi_\theta(u_i, x_i)$
        \State $\widehat{\VC}_{\varphi}(\theta) \gets \widehat{\VC}_{\varphi}(\theta) + \widehat{\psi}_i(\theta) + \widehat{\varphi}_i(\theta)$
      \EndFor
      \State Compute $\nabla_\theta \frac{1}{|\mathcal{B}|} \widehat{\VC}_\varphi(\theta)$ \Comment{Do not propagate gradients through $\check{u}$}
      \State Update $\theta$ with Adam
    \end{algorithmic}
  \end{algorithm}

  \begin{algorithm}[h!]
    \caption{Amortized Neural Vector Quantile Regression Training (AC-NQR)}  \label{alg:amoritzed_neural_quantile_regression_training}
    \begin{algorithmic}[1]
      \State \textbf{Input:} dataset $\{(x_i, y_i)\}_{i=1}^n$, PICNN $\varphi_\theta\colon \UC \times \XC \to \RR$, $\tilde{u}_{\paramamor}(y, x)\colon \YC \times \XC \to \UC$
      \State Sample mini-batch $\BC \subset \{1, \dots, n\}$
      \State Initialize $\VC_\varphi \gets 0, \VC_{\tilde{u}} \gets 0$
      \For{each $i \in \BC$}
        \State $\tilde{u}_i \gets \tilde{u}_{\paramamor}(y_i, x_i)$
        \State $\check{u}_i \gets \arg\max_{u \in \UC} J_{\varphi_\theta(\cdot,x_i)}(u, y_i)$ \Comment{Run L-BFGS for each $y_i$ starting at $u = \tilde{u}_i$}
        \State $\widehat{\psi}_i(\theta) \gets J_{\varphi_\theta(\cdot,x_i)}(\check{u}_i, y_i)$
        \State Sample $u_i \sim \NC(0,I_d)$
        \State $\widehat{\varphi}_i(\theta) \gets \varphi_\theta(u_i,x_i)$
        \State $\widehat{\VC}_{\varphi}(\theta) \gets \widehat{\VC}_{\varphi}(\theta) + \widehat{\psi}_i(\theta) + \widehat{\varphi}_i(\theta)$
        \State $\widehat{\VC}_{\tilde{u}}(\paramamor) \gets \widehat{\VC}_{\tilde{u}}(\paramamor) + \lVert \tilde{u}_i - \check{u}_i \rVert_2^2$ 
      \EndFor
      \State Compute $\nabla_\theta \frac{1}{|\mathcal{B}|}\widehat{\VC}_\varphi(\theta)$ and $\nabla_{\paramamor} \frac{1}{|\mathcal{B}|}\widehat{\VC}_{\tilde{u}}(\paramamor)$ \Comment{Do not propagate gradients through $\check{u}$}
      \State Update $\theta$ and $\paramamor$
    \end{algorithmic}
  \end{algorithm}

  \begin{algorithm}[h!]
    \caption{Entropic semi-dual training (EC-NQR)}
    \label{alg:entropic_semi_dual_training}
    \begin{algorithmic}[1]
      \State \textbf{Input:} dataset $\{(x_i, y_i)\}_{i=1}^n$, PICNN $\varphi_\theta\colon \UC \times \XC \to \RR$
      \State Sample mini-batch $\BC \subset \{1, \dots, n\}$
      \State initialize $\LC_\varphi \gets 0$
      \State Sample i.i.d. $u_{ij} \sim F_U$
      \For{each $i \in \BC$}
        \State $\widehat{\psi}_i(\theta) \gets \epsilon \log \sum_{j=1}^m\text{exp}\left(\frac{u_{ij}^Ty_i - \varphi_\theta(u_{ij}, x_i)}{{\epsilon}} \right)$ \Comment{$\varepsilon\in[10^{-3},10^{-1}]$}
        \State Sample $u_i \sim F_U$
        \State $\widehat{\varphi}_i(\theta) \gets \varphi_{\theta}(u_i, x_i)$; 
        \State $\LC_{\varphi}(\theta) \gets \LC_{\varphi}(\theta) + \widehat{\psi}_i(\theta) + \widehat{\varphi}_i(\theta)$
      \EndFor
      \State Compute $\nabla_\theta \frac{1}{|\mathcal{B}|}\LC_\varphi(\theta)$
      \State Update $\theta$ with Adam
    \end{algorithmic}
  \end{algorithm}

\subsection{Conformal Methods Implementation}
\label{sec:impl:conformal}
  Here, we provide a detailed description of our implementation of the methods introduced in Section~\ref{sec:conformal_ot_neural_maps}. For all proposed approaches, we start with an estimate $\widehat{Q}^{-1}_{Y \mid X}(y, x)$ that we obtain using a training set $\DC_\mathrm{train}$. All conformal methods operate on a separate held-out calibration set $\DC_\mathrm{cal}$. Since we need to replicate our uncertainty estimation experiments for multiple splits and datasets, we use the Amortized Neural Vector Quantile Regression version of our algorithm.

\ecoparagraph{Split Conformal Prediction with Monge-Kantorovich ranks.}
  An instance of classical split conformal prediction using a score derived from our vector quantile regressor.

  \begin{algorithm}[h!]
    \caption{Pull-back split conformal prediction}
    \label{alg:conformal_pb}
    \begin{algorithmic}[1]
      \State \textbf{Input:} dataset $\DC_\cal = \{(X_i, Y_i)\}_{i=1}^n$, trained model $\widehat{Q}^{-1}_{Y \mid X}(y, x)$, a new test point $(X_\test, Y_\test)$ and the desired nominal miscoverage level $\alpha$
      \For{each $i \in \{1, \dots, n\}$}
        \State $U_i \gets \widehat{Q}^{-1}_{Y \mid X}(Y_i, X_i)$
        \State $S_i \gets \|U_i\|$ 
      \EndFor
      \State $\rho_{1-\alpha} \gets \lceil (n+1)(1-\alpha)\rceil$-th largest $S_i$
      \State $\hat{\CC}^{\pullback}_{\alpha}(X_\test) \gets \left\{y\colon \|\widehat{Q}^{-1}_{Y \mid X}(y, X_\test)\| \le \rho_{1-\alpha} \right\} $
    \end{algorithmic}
  \end{algorithm}

\ecoparagraph{Fixed re-ranking.}
  To account for the misspecification of $\widehat{Q}^{-1}_{Y \mid X}(y, x)$ we introduce an intermediate re-ranking of $U_i$. We follow the approach of~\cite{thurin2025optimal}, but instead of a separate base model, we directly use our estimate: $S_i = U_i = \widehat{Q}^{-1}_{Y \mid X}(Y_i, X_i) \in \RR^{d_y}$. We divide our calibration set into two parts: the first part is used to estimate an OT map $\mathbf{R}\colon \UC \to \UC'$ and the second part is used to conformalize the result. In our experiments, we follow the original authors' approach and use $\UC' = \mathrm{U}(S^{d_y-1})$ - uniform distribution on the unit ball. To evaluate the map $\widehat{\mathbf{R}}$ on the new point, we map it to the corresponding closest point from the first calibration part.

  \begin{algorithm}[h!]
    \caption{Re-ranked pull-back split conformal prediction}
    \label{alg:conformal_rerank}
    \begin{algorithmic}[1]
      \State \textbf{Input:} dataset $\DC_\cal = \{(X_i, Y_i)\}_{i=1}^{n=n_1+n_2}$, trained model $\widehat{Q}^{-1}_{Y \mid X}(y, x)$, a new test point $(X_\test, Y_\test)$ and the desired nominal miscoverage level $\alpha$
      \For{each $i \in \{1, \dots, n_1\}$}
        \State $U_i \gets \widehat{Q}^{-1}_{Y \mid X}(Y_i, X_i)$
      \EndFor
      \State Estimate $\widehat{\mathbf{R}}$ using sample $(\{U_i\}_{i=1}^{n_1}, \{U'_i\}_{i=1}^{n_1})$ \Comment{$\{U'_i\}_{i=1}^{n_1}$ - reference sample from $\UC'$}
      \For{each $j \in \{1, \dots, n_2\}$}
        \State $S_j \gets \left\|\widehat{\mathbf{R}}\left(\widehat{Q}^{-1}_{Y \mid X}(Y_j, X_j)\right)\right\|$
      \EndFor      
      \State $\rho_{1-\alpha} \gets \lceil (n_2+1)(1-\alpha)\rceil$-th largest $S_j$
      \State $\hat{\CC}^{\unipullback}_{\alpha}(X_\test) \gets \left\{y\colon \left\|\widehat{\mathbf{R}}\left(\widehat{Q}^{-1}_{Y \mid X}(y, X_\test)\right)\right\| \le \rho_{1-\alpha} \right\} $
    \end{algorithmic}
  \end{algorithm}

  We use the code of~\citet{thurin2025optimal} to estimate $\widehat{\mathbf{R}}$ (we divide the original calibration set into two equal parts). This implementation uses the renowned POT library~\citep{flamary2021pot}, which provides efficient implementations of the various optimal transport techniques.

\subsection{Hyperparameters and Default Configuration}
\label{sec:impl:hparams}
  \begin{table}[t!]
  \label{table:hyperparams}
    \caption{Model hyperparameters for different datasets.}
    \centering
    \begin{tabular}{lccc}
      \toprule
      \textbf{Dataset(s)} & \textbf{Layer width} & \textbf{Layer depth} & \textbf{Batch size} \\
      \midrule
      \texttt{bio} & 12 & 4 & 512 \\
      \texttt{blog} & 16 & 4 & 512 \\
      \texttt{sgemm} & 46 & 4 & 8192 \\
      \texttt{scm20d} & 10 & 1 & 2048 \\
      \textit{Banana}, \textit{Convex Banana}, \textit{Star}, \textit{Convex Star} & 18 & 8 & 256 \\
      \textit{Glasses}, \textit{Convex Glasses}, \textit{Funnel} & 18 & 8 & 256 \\
    \bottomrule
    \end{tabular}
  \end{table}

  \begin{itemize}
    \item \textbf{Network sizes.} We typically use around $10\%$ of available data as parameters scale. See \Cref{table:hyperparams} for details.
    
    \item \textbf{Optimization.} AdamW (LR $10^{-2}$, weight decay $10^{-4}$). We use cosine warm restart for amortization network every $5$k–$10$k steps; We clip gradients at $1.0$.
    \item \textbf{Inner solver.} L-BFGS with Wolfe line search, $K_{\max}=50$ (amortized) or $100$ (no amortization); tolerance $10^{-5}$; domain projection when $\UC$ is bounded.
    \item \textbf{Amortizer.} Amortization network copies the potential network architecture in all our experiments.

    \item \textbf{Entropic.} In all our experiments we fix $\varepsilon=0.001$; $m=1024$ Monte Carlo samples per $(x,y)$.
  \end{itemize}
  These defaults matched the settings used across \Cref{sec:optimal_transport_experiments} and \Cref{sec:conformal_experiments} (metrics and datasets).

\section{Deferred Content for Conformal Prediction}
\label{sec:supp:conf_theory}

 We now proceed to provide the deferred content from \Cref{sec:conformal_ot_neural_maps}. We start by restating \Cref{thm:radial_volume_optimality} and its proof. Then, we showcase a setting where the assumptions of \Cref{thm:radial_volume_optimality} are met. Finally, we explain how the OT maps $Q_{Y|X}$ and $Q^{-1}_{Y|X}$ may be used to construct conformal sets using density estimation.

 \begin{theorem}[Volume–optimality of pullback balls under radiality]
    Fix $x \in \XC$ and reference distribution $F_U(u) = \phi(\|u\|)$ for a continuous $\phi\colon [0,\infty) \to (0,\infty)$ on $\UC$, under the assumptions of \Cref{thm:carlier-existence}, let $Q_{Y|X}$ and $Q^{-1}_{Y|X}$ be the vector quantile and multivariate rank functions. Assume that there exists $j_x$ such that for all $y$ in the support of $F_{Y|X}$, it holds
    \begin{equation*}
      \det \closed{\nabla_y Q^{-1}_{Y|X}(y, x)} = j_x\open{\|Q^{-1}_{Y|X}(y, x)\|},
    \end{equation*}
    and the function $r \mapsto \phi(r) \, j_x(r)$ is strictly decreasing. Let $r_\alpha > 0$ be the unique radius satisfying $\mu(B_{r_\alpha}) = 1 - \alpha$, where $\mu$ is the law corresponding to $F_U$ and $B_r=\{u\colon \|u\| \le r\}$. Define the pullback ball $\CC^\pullback_\alpha(x)\coloneqq \ens{y\colon \|Q^{-1}_{Y|X}(y, x)\| \leq r_\alpha}$. Then, $\CC^\pullback_\alpha(x)$ minimizes Lebesgue volume among all sets with $x$-conditional coverage of at least $1 - \alpha$, i.e., for every measurable $A \subset \YC_x$ with $\PP\{Y \in A \mid X = x\} \ge 1 - \alpha$, $\vol\bigl(\CC^\pullback_\alpha(x)\bigr) \le \vol(A)$.
  \end{theorem}

 \begin{proof}
    Let $S_x(\cdot) = Q^{-1}_{Y|X}(\cdot)$. Then, by the change of variables formula for densities:
    \[
      f_{Y\mid X}(y,x)
      = f_U\big(S_x(y)\big)\,\left|\det\bigl[\nabla_y S_x(y)\bigr]\right|.
    \]
    Using the assumption that $f_U(u) = \phi(\|u\|)$ and    $\det\big[\nabla_y S_x(y)\big]=j_x\big(\|S_x(y)\|\big)$.     Using \citet[Corollary~2.1]{carlier2016vector}, we note that $S_x$ is $C^1$ and the derivative of a convex function. Thus, it holds that $y\rightarrow \det \closed{\nabla_y S_x(y)}$ is positive and continuous, which allow for dropping absolute value to recover
    \[
      f_{Y\mid X}(y,x)
      = \phi\big(\|S_x(y)\|\big)\,j_x\big(\|S_x(y)\|\big)
      =: h_x\big(\|S_x(y)\|\big).
    \]
    As both $\phi$ and $y\rightarrow j_x(\|S_x(y)\|)$ are continuous, $h_x$ is a strictly decreasing continuous invertible function. Hence, $f_{Y\mid X}(\cdot, x)$ is a non-increasing function of the $U$–radius $\|S_x(y)\|$ and its superlevel sets are pullbacks of Euclidean balls:
    for each $t>0$ there exists $r(t)\ge0$ such that
    \[
      \{y\colon f_{Y\mid X}(y,x)\ge t\}
      = \{y\colon h_x(\|S_x(y)\|)\ge t\}
      =\ens{y\colon \|S_x(y)\|\leq r(t)}.
    \]

    We first record the probability identity. For any Borel $A \subset \YC_x$,
    \[
      \PP\{Y \in A \mid X = x\} = \mu\open{\ens{S_x(y)|y\in A }}.
    \]
    Therefore $\PP\{Y \in \CC^\star_\alpha(x) \mid X = x\} = \mu\bigl(B_{r_\alpha}\bigr) = 1 - \alpha$.

    For volume optimality, note that since $f_{Y \mid X}(y,x) = h_x(\|S_x(y)\|)$ with $h_x$ non-increasing, every HPD superlevel set $\{y\colon f_{Y \mid X}(y,x)\ge t\}$ is (almost surely) a pullback set of the form $\ens{y| \|S_x(y)\|\leq r(t)}$. Choosing $t_\alpha$ so that $\PP\{Y \in \{f_{Y\mid X}(\cdot,x)\ge t_\alpha\}\mid X=x\} = 1 - \alpha$ forces $\mu(B_{r(t_\alpha)}) = 1 - \alpha$, hence $r(t_\alpha) = r_\alpha$ and the HPD set equals $\CC^\pullback_\alpha(x)$. 
\end{proof}

 \begin{remark}[Examples satisfying assumptions of \Cref{thm:radial_volume_optimality}]
Fix $x$. Let the reference be spherical with radial, strictly decreasing continuous density $f_U(u)=\phi(\|u\|)$. Suppose $Y\mid X=x$ is elliptical with location $m(x)$ and a positive definite scatter matrix $\Sigma(x)$ whose whitened density uses the same radial generator as $U$, i.e., 
\[
f_{Y\mid X=x}(y)\ \propto\ \phi \left(\left\|\Sigma(x)^{-1/2}\big(y-m(x)\big)\right\|\right).
\]
Then the map $S_x(y)=\Sigma(x)^{-1/2}\big(y-m(x)\big)$ and $\det\big[\nabla_y S_x(y)\big]\equiv \det\big(\Sigma(x)^{-1/2}\big)$. This setting includes the Gaussian case by taking   $\phi(r)\propto e^{-r^2/2}$.

To show that $S_x$ is indeed the optimal transport map, note that $S_x$ is the gradient of convex quadratic function. Thus, it satisfies the Brenier optimal transport conditions for the Euclidean quadratic cost and , by Knott–Smith optimality criterion, it is the vector quantile function \citep{knott1984optimal}.  
\end{remark}

 \ecoparagraph{Conformal HDP Sets using OT Parameterization.}
  While the CQR-like construction in \Cref{sec:conformal_ot_neural_maps} is robust and simple, its prediction sets are images of Euclidean spheres and thus topologically connected since, under \Cref{assum:density-reference} and \Cref{assum:absolute_continuity_cond_distr}, $Q^{-1}_{Y\mid X}$ is continuous by \citet[Corollary~2.1]{carlier2016vector}. This can be inefficient if for some $x\in \XC$, the true conditional distribution $F_{Y\mid X=x}$ is multimodal, for example a Gaussian mixture. To solve this problem, it is possible to construct prediction sets using the level sets of an estimated conditional density, which can naturally form disconnected regions.

  This approach utilizes the change-of-variables formula and leveraging $\widehat{Q}^{-1}_{Y \mid X}$ to  recover the plug-in conditional density estimator
  \[
    \widehat{p}(y \mid x) = f_U\bigl(\widehat{Q}^{-1}_{Y \mid X}(y, x)\bigr) \det\bigl[\nabla_y \widehat{Q}^{-1}_{Y \mid X}(y, x)\bigr].
  \]
  This estimator can then be used to define conformity scores. For each point $(Y_i, X_i)$ in the calibration set $\DC_{\mathrm{cal}}$ we calculate the score $s_i = \widehat{p}(Y_i \mid X_i)$. The prediction set for a new point $X_{\test}$ is the superlevel set of this estimated density, where the level is calibrated to ensure coverage. If $s_{(1)} \le \dots \le s_{(n)}$ are the ordered scores from the calibration set, we set the threshold $\tau = s_{(\lfloor (n+1)\alpha \rfloor)}$. Then, the HPD-style prediction region is given by:
  \[
    \CC^{\mathrm{hpd}}_\alpha(x) = \bigl\{y \in \YC\colon \widehat{p}(y \mid x) \ge \tau \bigr\}.
  \]
  By standard arguments, this set fulfills the marginal coverage guarantee $\PP_{(Y,X) \sim F_{YX}}(Y \in \CC^{\mathrm{hpd}}_\alpha(X)\} \ge 1-\alpha$. Crucially, if the learned map $\widehat{Q}^{-1}_{Y\mid X}$ recovers the true rank map, then $\widehat{p}(\cdot \mid x)$ recovers the true conditional density, and the resulting prediction set is exactly the true HPD region.

 \ecoparagraph{Related density–based approaches.} The idea of using density estimation to construct conformal sets has been exploited in recent related works. For example, in the setting with $\YC \subseteq \rset$, \textit{CD-split} partition $\XC$ into multiple splits, leverage a conditional density estimator $\hat{f}(y\mid x)$, and perform conformal calibration in split-wise manner to improve conditional coverage \citep{izbicki2022cd}. Furthermore, also in the setting with $\YC \subseteq \rset$, \textit{SPICE} learns a neural conditional density via deep splines and uses negative log-density/HPD scores to construct the conformal sets \citep{diamant2024conformalized}. 

\begin{remark} To construct conformal sets using density estimation, the estimator of $\widehat{p}(y\mid x)$ requires the Jacobian of $\hat{Q}^{-1}_{Y|X}$. Even if $\widehat{Q}^{-1}_{Y|X}$ approximates $Q^{-1}_{Y|X}$, $\nabla_y\widehat{Q}^{-1}_{Y|X}$ may not necessary approximate well $\nabla_y Q^{-1}_{Y|X}$.
Empirically, small errors in the Jacobian can be magnified in $\det(\cdot)$, which distorts HPD superlevel sets. As shown in \Cref{sec:optimal_transport_experiments}, in our experiments, $\widehat{Q}_{Y|X}$ approximated well the true quantile function. Nonetheless, we found the HDP approach of producing conformal sets empirically suboptimal w.r.t. the volume of the produced set and conditional coverage. 
\end{remark}

\section{Detailed Experimental Results}
\label{sec:supp:experiment_details}

\subsection{Optimal Transport Metrics} 
\label{sec:supp:optimal_transport_experiments_metrics}
\begin{itemize}
    \item \textbf{Wasserstein distances.} We compute Wasserstein-2 and Sliced Wasserstein distances using the \emph{POT} library \cite{flamary2021pot}.
    
    \item \textbf{KDE-L1.} To estimate the $L^1$ distance between kernel density estimators, we draw $1000$ samples from both $Q^{-1}_{Y \mid X}$ and its approximation $\widehat{Q}^{-1}_{Y \mid X}$. We then fit Gaussian kernel density estimates to each sample set and report the average pointwise $L^1$ difference between the two densities, evaluated at points drawn from $Q^{-1}_{Y \mid X}$.
    
    \item \textbf{KDE-KL.} The Kullback–Leibler divergence is computed following the same procedure as KDE-L1. We report the average pointwise KL divergence between the fitted densities at points drawn from $Q^{-1}_{Y \mid X}$.
    
    \item \textbf{L2-UV.} To compute the unexplained variance ratio, we sample $n_u$ points from $u_{\text{test}} \sim F_U$ and $n_x$ points from $x_{\text{test}} \sim F_X$. The L2-UV distance is then defined as
    \[
    \frac{1}{n_x + n_u} \sum_{x_{\text{test}}, u_{\text{test}}} 
    \frac{\lVert Q_{U \mid X}(u_{\text{test}}, x_{\text{test}})  - \widehat{Q}_{U \mid X}(u_{\text{test}}, x_{\text{test}}) \rVert_2}
         {\big\lVert \tfrac{1}{n_u}\sum_{u_{\text{test}}} Q_{U \mid X}(u_{\text{test}}, x_{\text{test}}) - Q_{U \mid X}(u_{\text{test}}, x_{\text{test}}) \big\rVert_2}.
    \]
\end{itemize}

\subsection{Optimal Transport experiments datasets} 
\label{sec:supp:optimal_transport_experiments_datasets}
\paragraph{Banana Dataset.}
  This dataset is largely used in vector quantile estimation for testing the non-linearity of estimators. It was introduced in~\citep{feldman2023calibrated} and used in~\citep{carlier2017vector,rosenberg2023fast}. It represents a banana-shaped random variable in $\RR^2$, changing its position and skewness based on latent random variable from $\RR^1$. Data generative process can be described as:
  \begin{align*}
    X &\sim \mathcal{U}[0.8, 3.2], \quad 
    Z \sim \mathcal{U}[-\pi, \pi], \quad 
    \varphi \sim \mathcal{U}[0, 2\pi], \quad 
    r \sim \mathcal{U}[-0.1, 0.1], \\
    \hat{\beta} &\sim \mathcal{U}[0, 1]^k, \quad 
    \beta = \frac{\hat{\beta}}{\|\hat{\beta}\|_1}, \\
    Y_0 &= \tfrac{1}{2} \left(-\cos(Z) + 1\right) + r \sin(\varphi) + \sin(X), \\
    Y_1 &= \frac{Z}{\beta X} + r \cos(\varphi), \\
    &\mathbf{X} = X, \mathbf{Y} = \begin{bmatrix} Y_0 \\ Y_1 \end{bmatrix}.
  \end{align*}
  We take $\mathbf{X}$ as and $\mathbf{Y}$ as observed random variables.

  Full set of metrics for Banana dataset is accessible at \cref{fig:banana_results}. Metrics for convex potential, that was trained on Banana dataset can be found at \cref{fig:convex_banana_results}.

\paragraph{Rotating Star.}
  This dataset is inspired by~\citep{rosenberg2023fast} rotating star example. Observed random variable represents a three point star in $\RR^2$ that rotates based on latent variable from $\RR$. Data generative process can be described as:
  \begin{align*}
    (u_0, u_1) &\sim \NC(0, I), \quad X \sim \UC\Big[0, \frac{2}{3}\Big], \\
    \theta &= \arctan\!\left(\tfrac{u_1}{u_0}\right), \quad s(\theta) = 1 + 3\cos(3 \theta), \\
    \mathbf{R}(\varphi) &= 
    \begin{bmatrix}
      \cos \varphi & -\sin \varphi \\
      \sin \varphi & \cos \varphi
    \end{bmatrix}, \\
    \mathbf{Y} &= \mathbf{R}(\varphi)\bigl( s(\theta) u_0, s(\theta) u_1 \bigr)^\top, \mathbf{X}=X,
  \end{align*}
  where $\varphi$ is a rotation angle. We take $\mathbf{X}, \mathbf{Y}$ as observed variables.
  
  Full set of metrics for Star dataset is accessible at \Cref{fig:star_results}. Metrics for convex potential, that was trained on Star dataset can be found at \Cref{fig:convex_star_results}
  
\paragraph{Glasses.}
  This dataset is introduced in~\citep{brando2022deep}. It represents two modal distribution, where random variable is in $\RR$. With $X\sim \UC[0, 1]$, data generative process can be described as:
  \begin{align*}
    z_{1} &= 3\pi X, \quad z_{2} = \pi (1 + 3X), \quad \epsilon \sim \text{Beta}(\alpha = 0.5, \beta = 1), \\
    Y_{1} &= 5 \sin(z_{1}) + 2.5 + \epsilon, \quad Y_{2} = 5 \sin(z_{2}) + 2.5 - \epsilon, \\
    \gamma &\sim \text{Categorical}(0, 1), \\
    \mathbf{Y} &= (1-\gamma)Y_1 + \gamma Y_2.
  \end{align*}
  We take $\mathbf{X}, \mathbf{Y}$ as observed variables.
  Full set of metrics for Glasses dataset is accessible at \Cref{fig:glasses_results}. Metrics for convex potential, that was trained on Glasses dataset can be found at \Cref{fig:convex_glasses_results}

\paragraph{Neal's funnel distribution.}
  The classical funnel distribution~\citep{neal2003slice} is defined on
  $\mathbb{R}^{d+1}$ as
  \[
    v \sim \mathcal{N}(0, \sigma^2),
    \qquad
    x_i \mid v \;\sim\; \mathcal{N}\!\big(0, \exp(v)\big),
    \quad i=1,\dots,d,
  \]
  so that the joint density of $(v, x_1,\dots,x_d)$ is
  \[
    p(v, x) =
    \frac{1}{\sqrt{2\pi\sigma^2}} \exp\left(-\frac{v^2}{2\sigma^2}\right)
    \prod_{i=1}^d \frac{1}{\sqrt{2\pi e^v}}
    \exp\left(-\tfrac{x_i^2}{2 e^v}\right).
  \]
  For large negative values of $v$, the conditional variance of the $x_i$’s
  shrinks, yielding a narrow region (the ``neck’’ of the funnel),
  whereas large positive $v$ produces very diffuse $x_i$’s (the ``mouth’’).
  This strong nonlinearity makes the distribution challenging for MCMC methods.
  
\paragraph{Multidimensional funnel.}
  A natural generalization introduces a $k$-dimensional scale vector
  $v = (v_1,\dots,v_k)$ with
  \[
    v_j \sim \mathcal{N}(0, \sigma^2),
    \qquad
    x_{j,\ell} \mid v_j \;\sim\; \mathcal{N}\!\big(0, \exp(v_j)\big),
    \quad \ell = 1,\dots,m,
  \]
  so that each $v_j$ controls a block of $m$ Gaussian variables.
  The joint distribution then lives in dimension $k(1+m)$ and exhibits
  multiple funnel directions simultaneously.
  This high-dimensional geometry is frequently used as a stress test
  for MCMC and normalizing flow methods.


\subsection{Detailed Results of the Conformal Prediction Experiments}
\label{sec:impl:conformal}
We present more detailed results on conditional coverage on real datasets, involving more variations of our methods and more nominal levels $\alpha$. 

\paragraph{Methods.} We include the HPD variant of our method as well as models estimating either the forward (U) or the inverse (Y) quantile map. 

For methods labeled with Y, we model the function $\psi$ with a neural network and have $\widehat{Q}^{-1}_{Y \mid X}(y, x) = \nabla_y\psi(y, x)$. For methods labeled with U we model function $\varphi$ and get $\widehat{Q}_{Y \mid X}(y, x) = \nabla_u\varphi(u, x)$. 

Method \textrm{Quantile} corresponds to using the Monge-Kantorovich rank to construct the predictive regions, assuming that we have found exactly the mapping to the reference standard multidimensional normal distribution. In this particular case, the squared ranks follow the Chi-square distribution and the corresponding radius for the construction of the pullback-type prediction set can be found exactly. 

The methods labeled with RF correspond to fitting our model to the residuals of $s = y - \hat{f}(x)$ of a base Random Forest predictor $\hat{f}$. Base predictor uses 25\% of the training data and remainder is used to train our model.

\begin{figure}[t!]
  \label{fig:big_wsc}
    \centering
    \includegraphics[width=\textwidth]{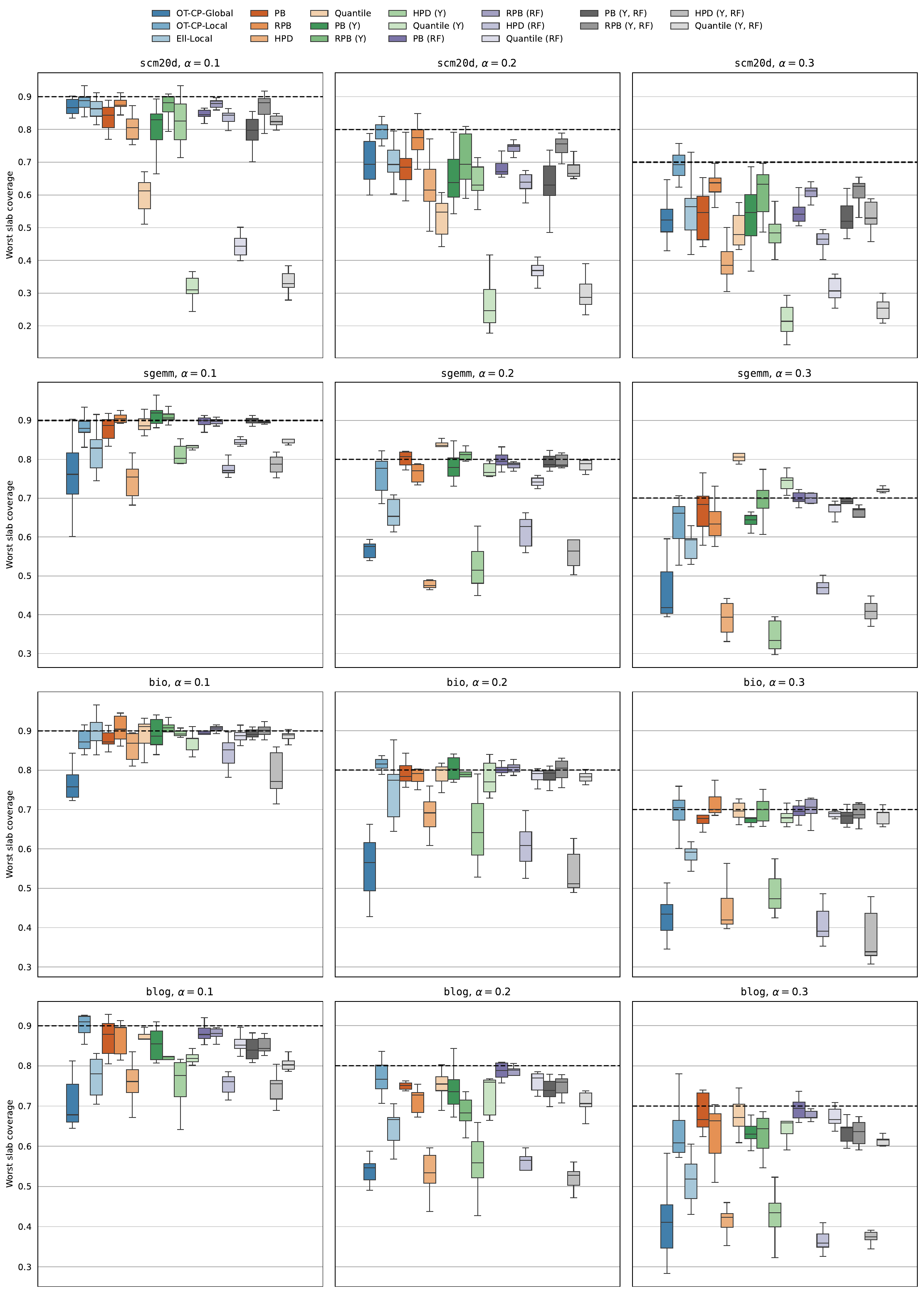}
    \caption{Worst slab coverage at different nominal miscoverage $\alpha$ levels for conformal prediction methods, achieved on large datasets.}
  \end{figure}

\paragraph{Implementation details.} For baseline methods we use the original authors implementation, where available and their suggested values for hyperparameters. For our methods, we select the number of parameters for neural networks to be roughly 10\% of the number of training samples. We tune the other hyperparameters for each dataset using a separate data split and utilize the mean coverage error of the pullback sets at different levels of $\alpha$ as a performance measure. All experiments were replicated using 10 random splits of the data into training, calibration, and test parts.

\paragraph{Discussion.} The Quantile method fails to achieve the nominal levels of conditional coverage, which suggests that a supporting measure like conformal prediction is indeed required. Unfortunately, HPD approaches do not perform well on many occasions, proving that density estimation in multiple dimensions is still a difficult to solve problem.

Using a base model and fitting quantile regression to the residuals instead of directly $Y$ provides less variable results, but does not always improve performance of our methods.

   \begin{figure}[t]
    \centering
    \includegraphics[width=\linewidth]{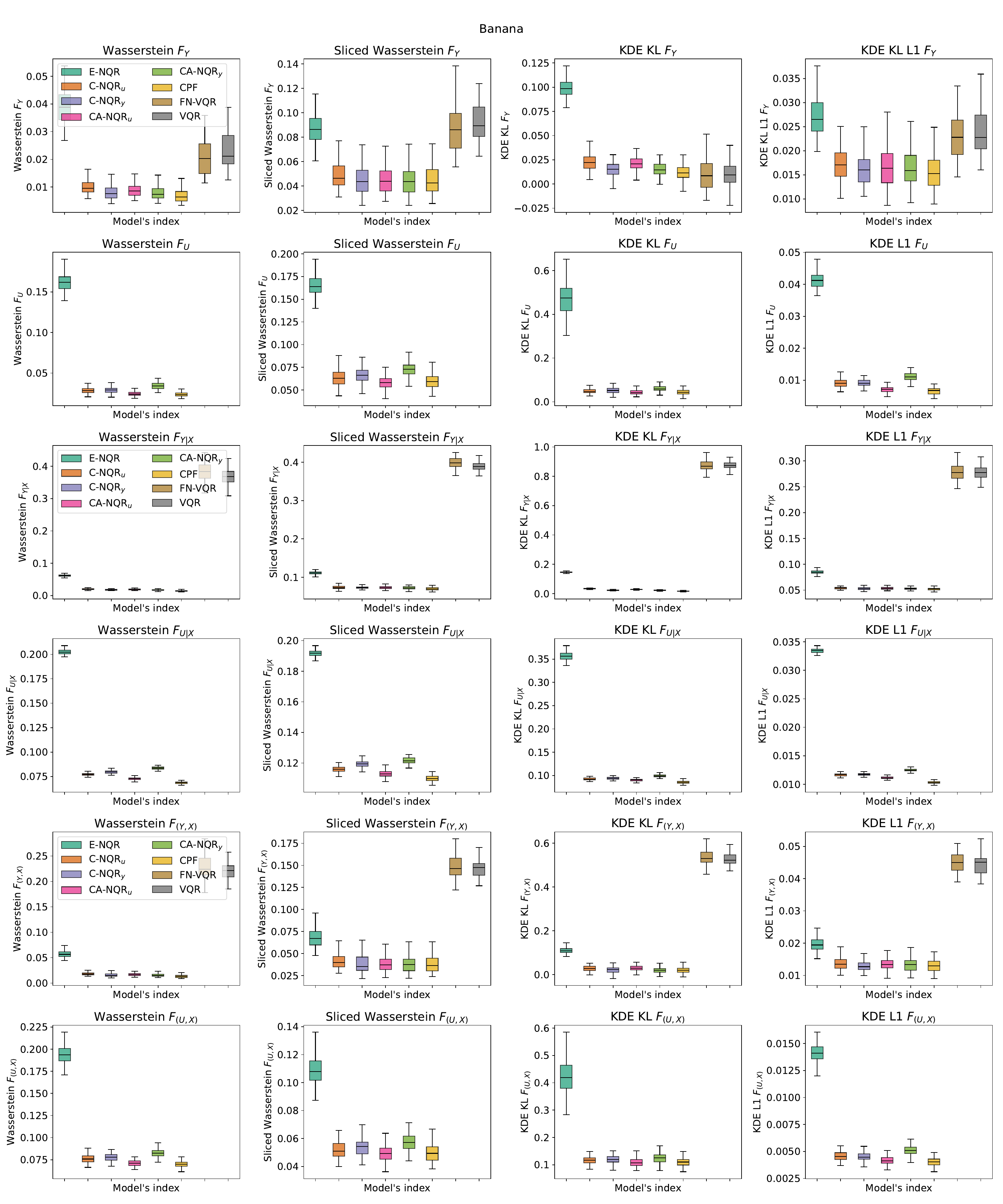}
    \vspace{-20pt}
    \caption{Full set of metrics for Banana dataset.}
    \label{fig:banana_results}
  \end{figure}

  \begin{figure}[t]
    \centering
    \includegraphics[width=\linewidth]{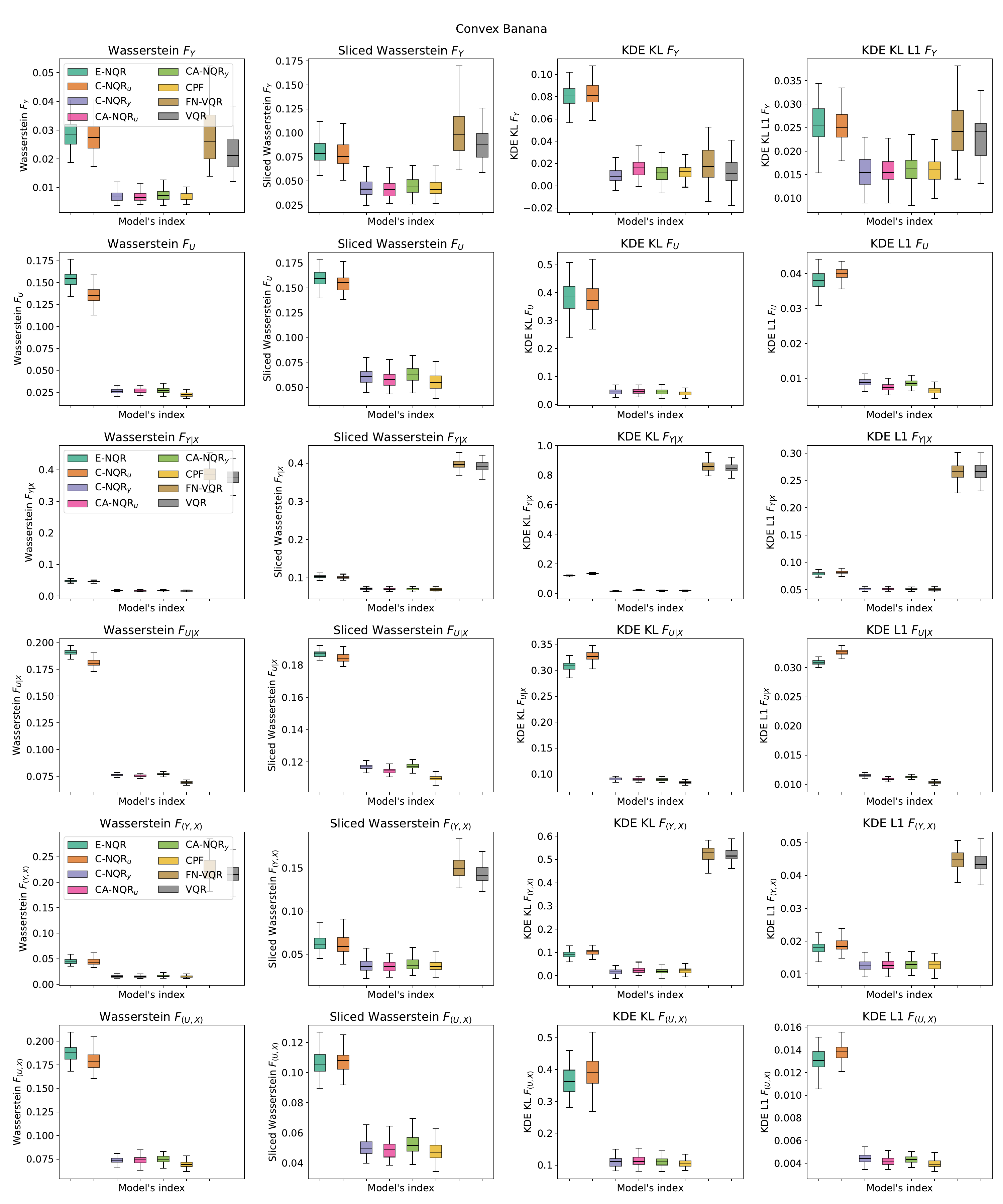}
    \vspace{-20pt}
    \caption{Full set of metrics for Banana dataset.}
    \label{fig:convex_banana_results}
  \end{figure}
\begin{figure}[t]
    \centering
    \includegraphics[width=\linewidth]{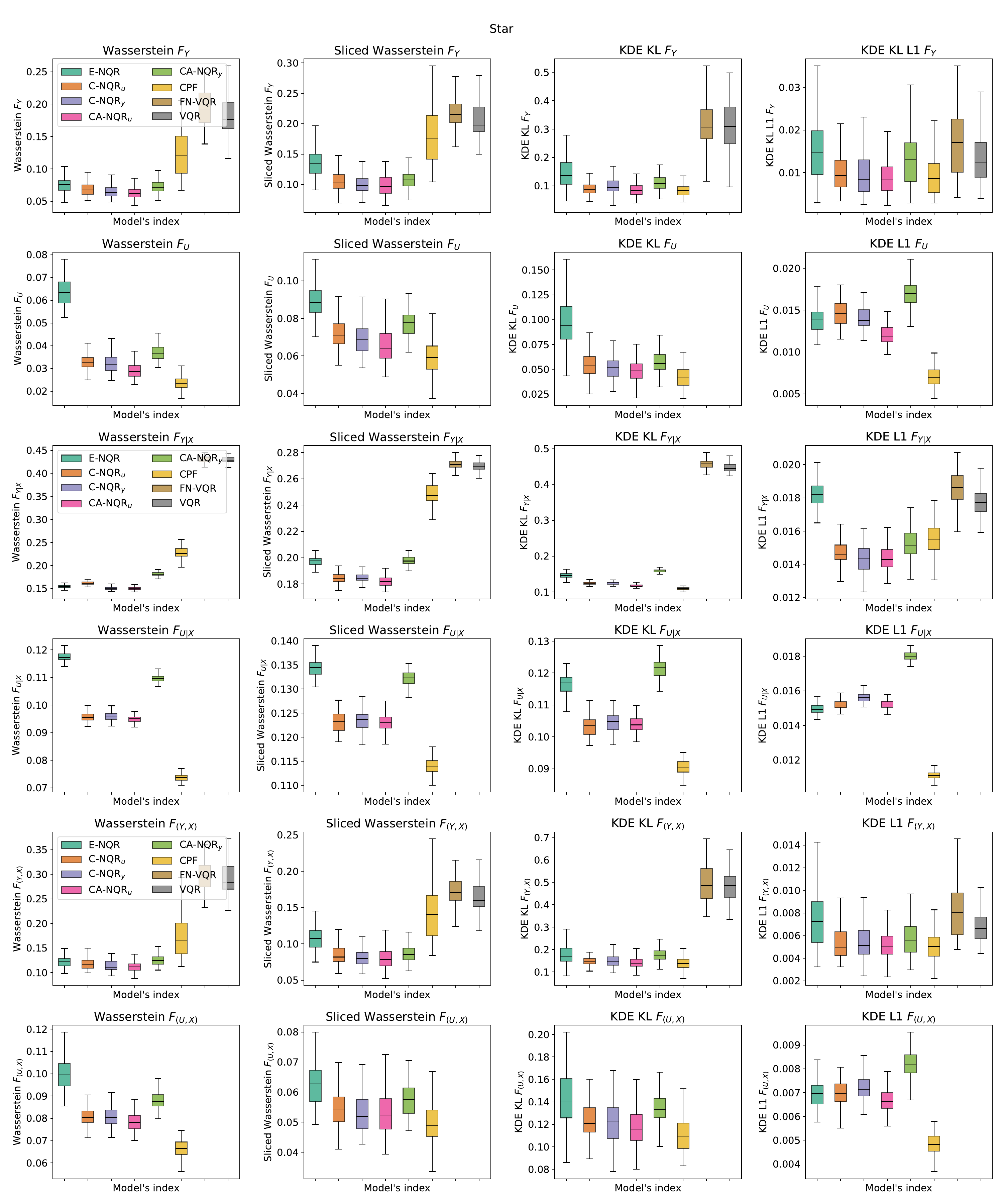}
    \vspace{-20pt}
    \caption{Full set of metrics for Star dataset.}
    \label{fig:star_results}
  \end{figure}

  \begin{figure}[t]
    \centering
    \includegraphics[width=\linewidth]{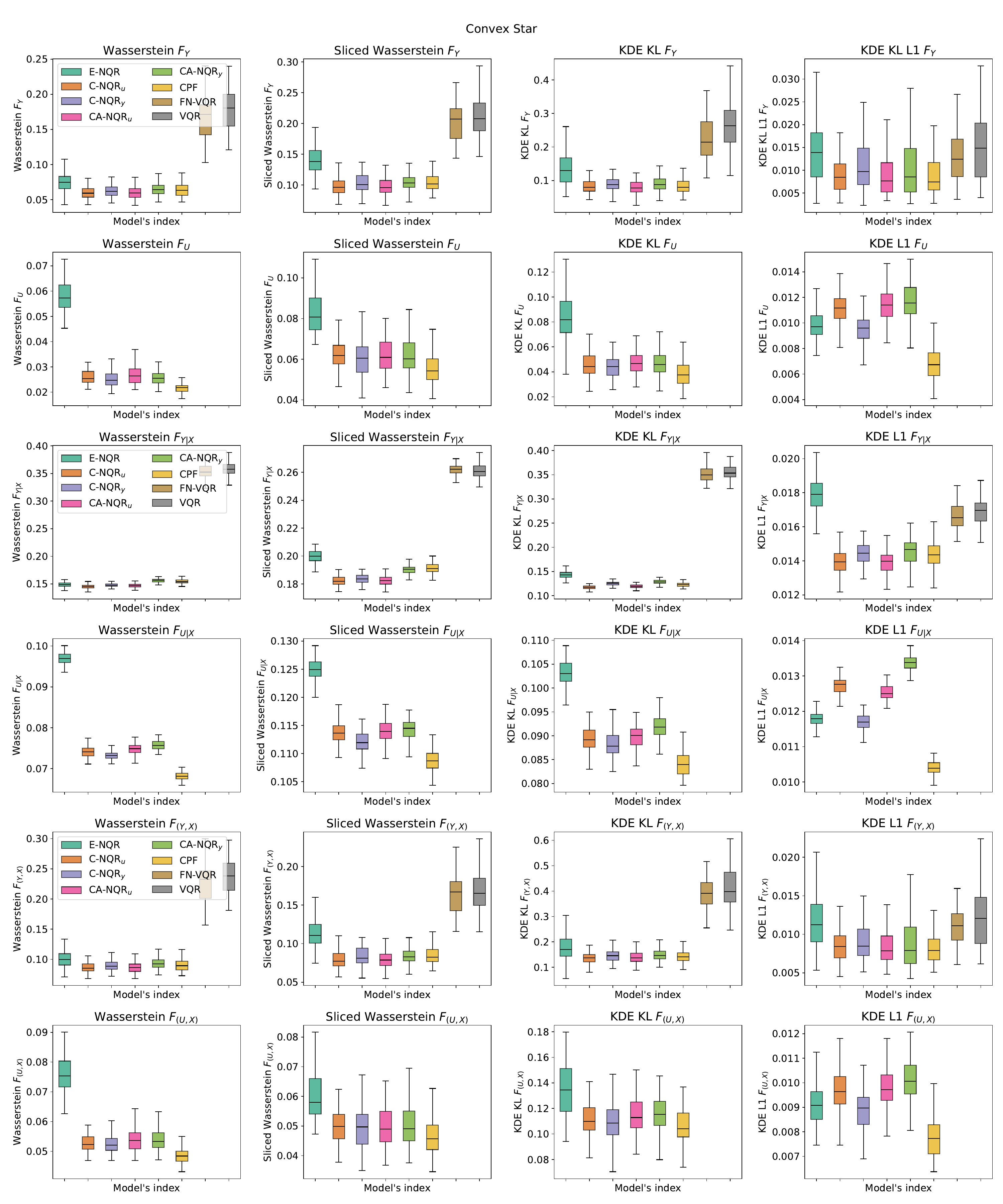}
    \vspace{-20pt}
    \caption{Full set of metrics for Convex Star dataset.}
    \label{fig:convex_star_results}
  \end{figure}
  
  \begin{figure}[t]
    \centering
    \includegraphics[width=\linewidth]{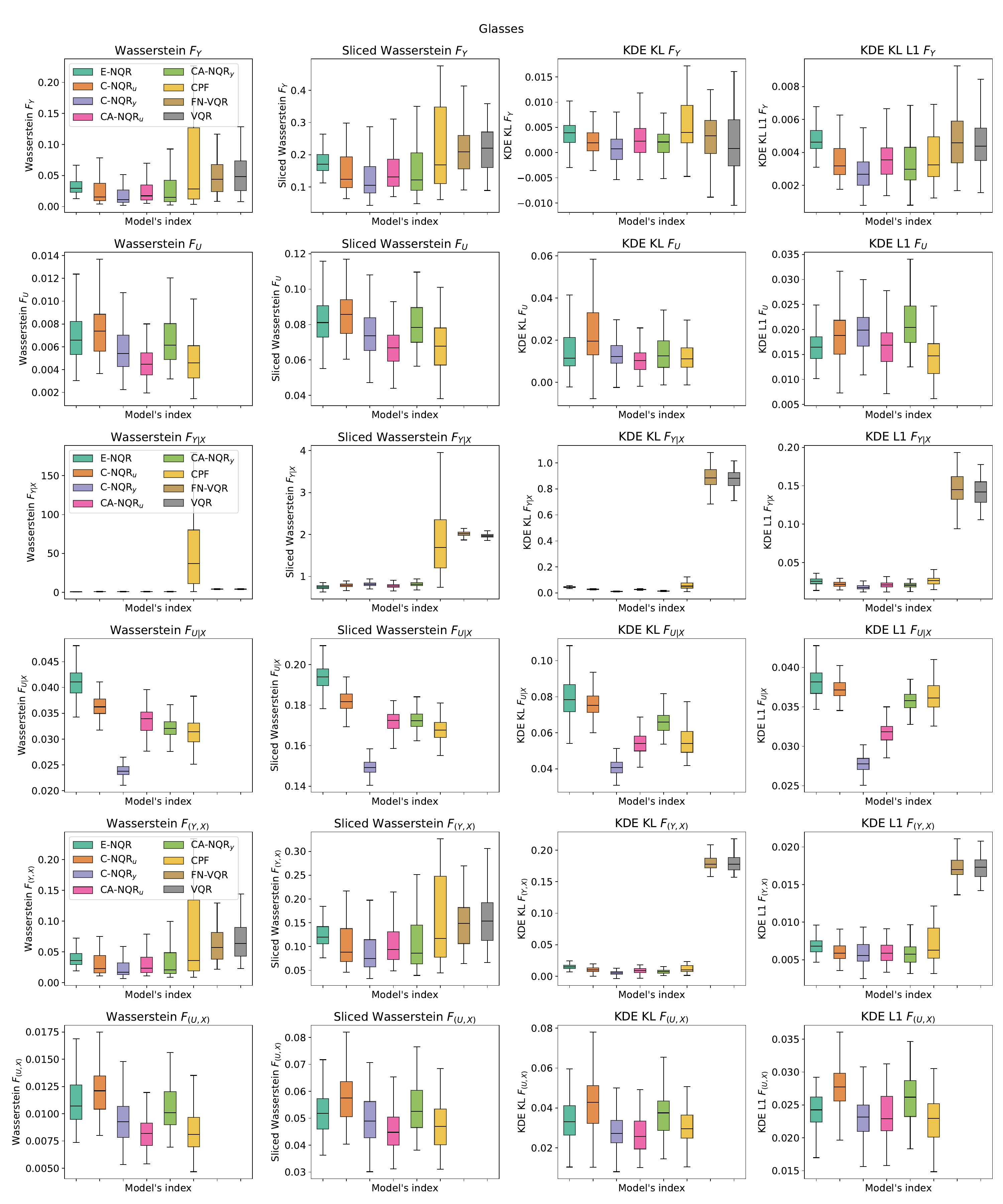}
    \vspace{-20pt}
    \caption{Full set of metrics for Glasses dataset.}
    \label{fig:glasses_results}
  \end{figure}

  \begin{figure}[t]
    \centering
    \includegraphics[width=\linewidth]{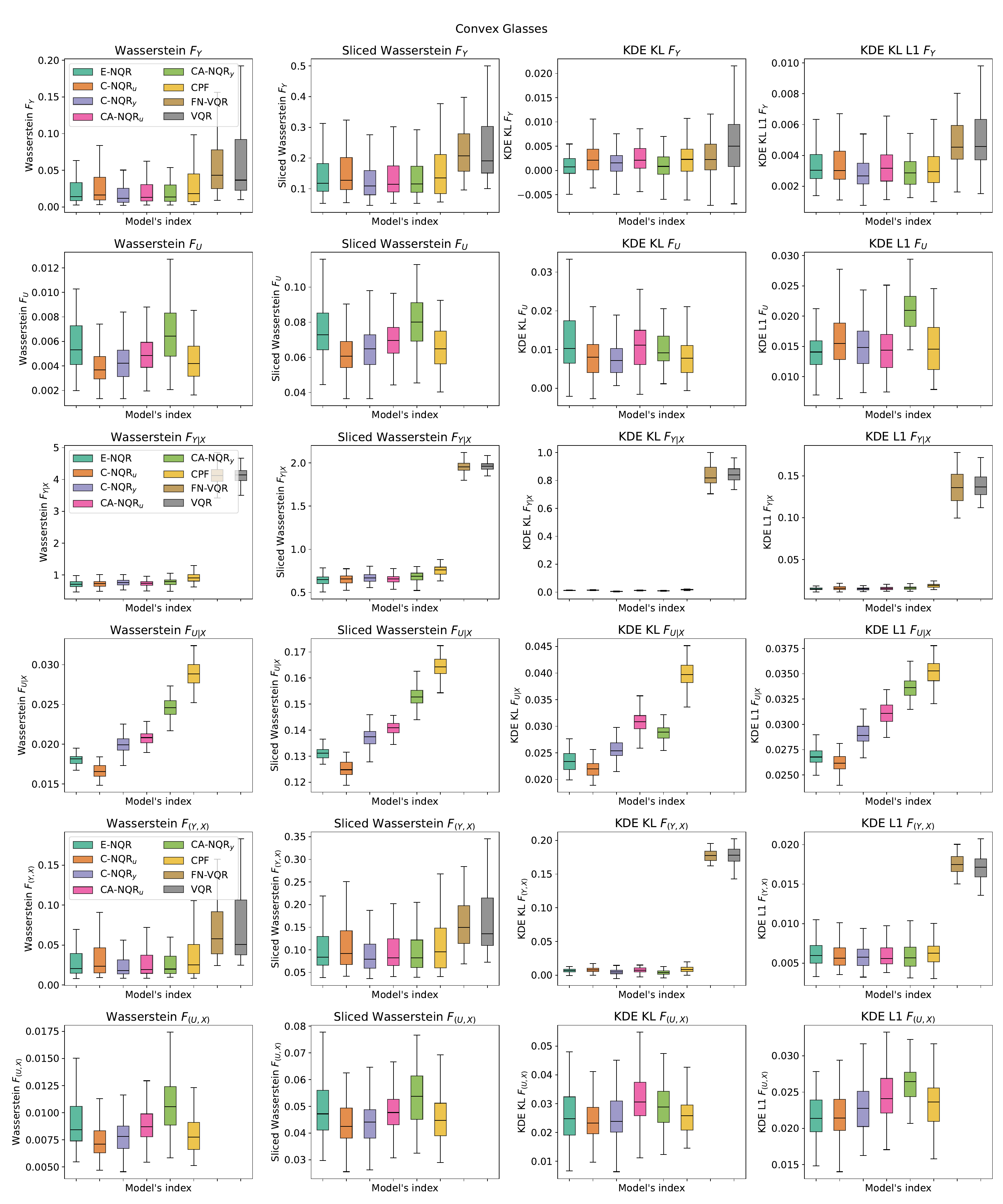}
    \vspace{-20pt}
    \caption{Full set of metrics for Convex Glasses dataset.}
    \label{fig:convex_glasses_results}
  \end{figure}

\end{document}